\documentclass[10pt,twocolumn,letterpaper]{article}

\usepackage[pagenumbers]{cvpr} %
\usepackage[utf8]{inputenc} %
\usepackage[T1]{fontenc}    %
\usepackage{booktabs}       %
\usepackage{amsfonts}       %
\usepackage{nicefrac}       %
\usepackage{microtype}      %
\usepackage{xcolor}         %
\usepackage{colortbl}       %
\usepackage{tcolorbox}     %

\usepackage{graphicx}
\usepackage{enumitem}
\usepackage{amsmath}
\usepackage{multirow}
\usepackage{amsthm}

\theoremstyle{plain}
\newtheorem{theorem}{Theorem}[section]
\newtheorem*{theorem*}{Theorem}

\newtheorem{lemma}[theorem]{Lemma}
\newtheorem*{lemma*}{Lemma}

\theoremstyle{definition}
\newtheorem{definition}[theorem]{Definition}

\theoremstyle{remark}

\newcommand{\bs}[1]{\boldsymbol{#1}}

\DeclareRobustCommand{\algnameformat}[1]{%
    \ifdim\fontdimen1\font>0pt\relax
        #1%
    \else
        \textsc{#1}%
    \fi
}
\DeclareRobustCommand{\algname}{\algnameformat{ORIGIN}}

\definecolor{cvprblue}{rgb}{0.21,0.49,0.74}
\usepackage[pagebackref,breaklinks,colorlinks,allcolors=cvprblue]{hyperref}
\def\confName{CVPR}
\def\confYear{2027}

\title{From Inference to Adaptation: Vision-Language Model Test-Time Adaptation via a Unified Optimal Transport View}
\makeatletter
\patchcmd{\@maketitle}{\vspace*{24pt}}{\vspace*{24pt}\vspace{-10pt}}{}{}
\patchcmd{\@maketitle}{\vspace*{12pt}}{\vspace*{12pt}\vspace{-8pt}}{}{}
\makeatother
\author{%
  Qi Yu$^1$ ~~
  Zhichen Zeng$^1$ ~~
  Katherine Tieu$^1$ ~~
  Lingjie Chen$^1$ ~~
  Xiyuan Yang$^1$ ~~
  Ruizhong Qiu$^1$ \\
  Yuchen Yan$^2$ ~~
  Lihui Liu$^3$ ~~
  Yanjun Zhao$^1$ ~~
  Ziao Wang$^4$ ~~
  Lei Ying$^4$ ~~
  Jingrui He$^1$ ~~
  Hanghang Tong$^1$ \\
  $^1$ \textbf{University of Illinois Urbana-Champaign} ~~
  $^2$ \textbf{Amazon} \\
  $^3$ \textbf{Wayne State University} ~~
  $^4$ \textbf{University of Michigan, Ann Arbor} ~~
}

\begin{document}
\maketitle

\begin{abstract}
    Vision-language models (VLMs) have demonstrated remarkable zero-shot capabilities yet remain sensitive to real-world distribution shifts.
    Although significant efforts are devoted to adapting VLMs at test time,
    they rely heavily on noisy pseudo-labels predicted from cross-modal embedding similarities during inference,
    which are unreliable under distribution shifts and mislead the adaptation.
    Moreover, existing methods largely overlook the critical relational structures across different modalities during test-time adaptation,
    whose supervision signals remain surprisingly robust under distribution shifts.
    In this work, we aim to jointly align the cross-modal semantics and intra-modal relational structures across modalities at test time to improve VLM adaptation.
    At the inference stage, we formulate the zero-shot image classification task as a cross-modal alignment problem via a Fused Gromov-Wasserstein (FGW) formulation,
    which considers both cross-modal semantics and intra-modal structures to generate robust pseudo-labels.
    At the adaptation stage, we propose a multi-level InfoNCE loss to leverage both the semantical and relational alignment supervision from the inference stage,
    empowering fine-grained adaptation of VLMs under distribution shifts with additional structural consistency.
    Finally, we bridge the inference and adaptation stages of our method by proving a mathematical equivalence between their objective functions,
    thereby revealing the mutual benefits of these two stages from a theoretical perspective.
    Extensive experiments demonstrate the effectiveness and efficiency of our methods,
    outperforming the best-performing methods by up to 3\% while maintaining a superior balance between effectiveness and efficiency.
\end{abstract}

\vspace{-15pt}
\section{Introduction}\label{sec:intro}
Vision-language models (VLMs), such as CLIP~\citep{radford2021learning},
learn to project image and text embeddings into a shared embedding space,
and have achieved remarkable zero-shot capability across diverse multimodal learning tasks,
such as image classification~\citep{radford2021learning}, image captioning~\citep{li2022blip}, and semantic segmentation~\citep{luddecke2022image}.
Despite the success of pre-trained VLMs, the aligned embedding spaces are vulnerable to real-world distribution shifts,
such as image corruption~\citep{hendrycks2019benchmarking}, causing significant performance degradation of VLMs at deployment.
To mitigate this issue, a predominant approach is to adapt pre-trained VLMs at test time, i.e., test-time adaptation (TTA), based on unlabeled test data.
This paradigm allows pre-trained VLMs to adapt dynamically to shifted domains at deployment without accessing large-scale source training data or test labels.

\begin{figure}[t]
  \includegraphics[width=\linewidth, trim=10bp 5bp 10bp 10bp, clip]{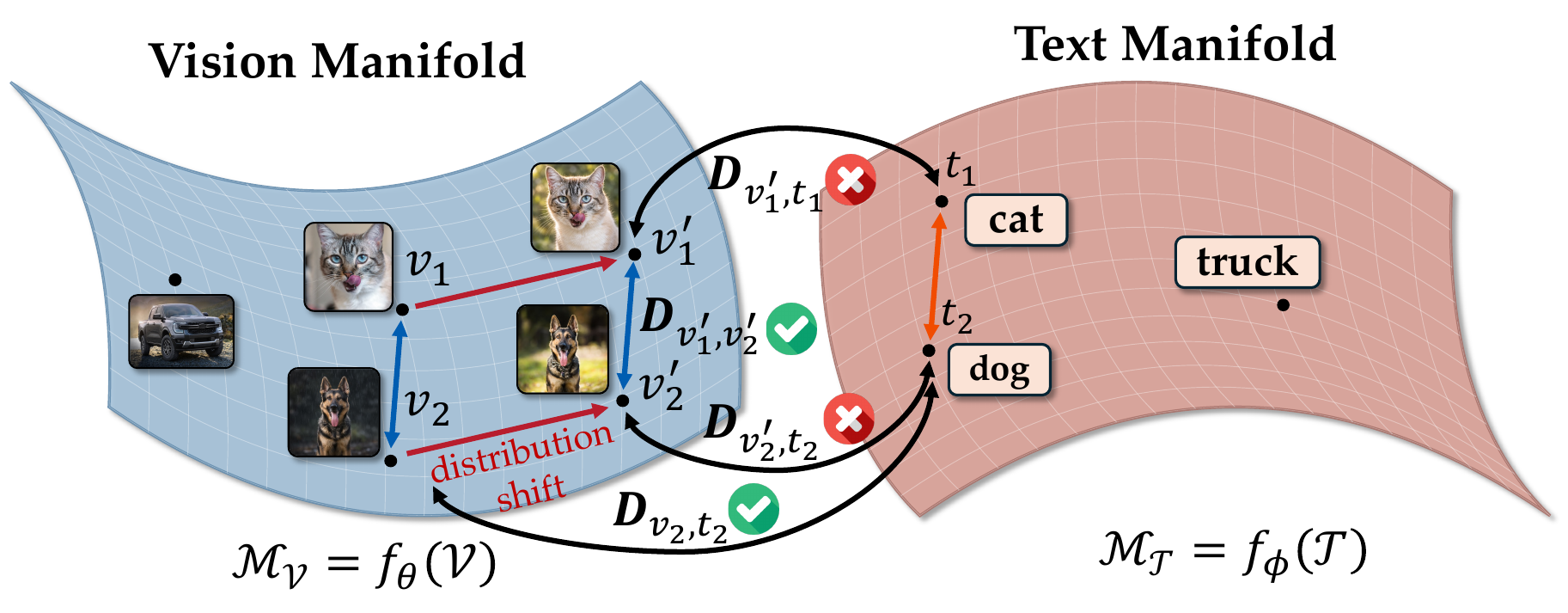}
  \caption{An illustrative example of vision and text manifolds of VLMs under distribution shifts. Under distribution shifts,
    \textit{cross-modal semantics}, represented by the embedding distance $\mathbf{D}_{v'_i,t_i}$ between individual image $v'_i$ and text $t_i$, is largely distorted.
    On the other hand, \textit{intra-modal relational structures}, represented by the embedding distances $\mathbf{D}_{v'_i,v'_j}$ between two images $v'_i$ and $v'_j$,
    remain relatively stable.}
  \vspace{-10pt}
  \label{fig:teaser}
\end{figure}

Although existing TTA methods have achieved strong performance under distribution shifts,
they bear the following two key limitations which lead to suboptimal adaptation performance.
Firstly \emph{(Limitation \#1)}, existing TTA methods predominantly rely on pairwise image-text embedding similarities to predict pseudo-labels for VLM adaptation~\citep{shu2022tpt, li2024tda, mishra2024semantic, sheng2025illusion, lafon2026cliptta}.
However, such direct cross-modal semantics become unreliable under distribution shifts, as illustrated in Figure~\ref{fig:teaser}, leading to noisy zero-shot predictions that misguide the adaptation process.
We further quantify the robustness of cross-modal semantics in Figure~\ref{fig:spearman-panel}, which shows a severe degradation under distribution shifts.
Secondly \emph{(Limitation \#2)}, the adaptation of VLMs focuses exclusively on re-aligning individual image and text embeddings pairwise~\cite{wu2021data, mishra2024semantic, lafon2026cliptta},
overlooking the critical \textit{relational structure} across different modalities, i.e., the correspondence between intra-modal image-image relations and text-text relations.
However, such relational structures capture important intra-modal consistency beyond pairwise cross-modal pairs, which are critical to the performance of multimodal foundation models~\citep{li2026rethinking, wang2026mirror}.
As shown in Figure~\ref{fig:rel-hist}, correct VLM predictions consistently exhibit smaller relational discrepancy under different shift levels, revealing the critical role of relational structures in improving the performance of VLMs at test time.

In this paper, we propose a unified VLM TTA framework which leverages optimal transport (OT) to jointly align the cross-modal image-text semantics and their intra-modal relational structures across two modalities at test time.
At the inference stage, we formulate the zero-shot prediction as a cross-modal alignment task based on a Fused Gromov-Wasserstein (FGW) formulation,
explicitly matching the relational structures across two modalities on top of pairwise image-text alignment.
As shown in Figure~\ref{fig:teaser} and Figure~\ref{fig:spearman-panel},
intra-modal embedding consistency remains largely stable under distribution shifts relative to cross-modal embedding similarities,
thus providing robust matching signals to generate more accurate pseudo-labels \emph{(Limitation 1)}.
At the adaptation stage, we propose a multi-level InfoNCE loss that combines a semantic InfoNCE loss and a relational InfoNCE loss to leverage both the semantic and relational alignment signals for VLM adaptation,
encouraging the VLM to capture both the cross-modal semantics and the intra-modal relational structures during adaptation \emph{(Limitation 2)}.
Moreover, we reveal theoretically that the multi-level InfoNCE loss at the adaptation stage can be neatly unified with the entropic FGW formulation at the inference stage,
thereby bridging the objective function of inference and adaptation of VLMs at test time to achieve their mutual benefits.

\begin{figure}[t]
  \centering
  \begin{subfigure}[b]{0.333333\linewidth}
    \centering
    \includegraphics[width=\linewidth, trim=5bp 5bp 5bp 5bp, clip]{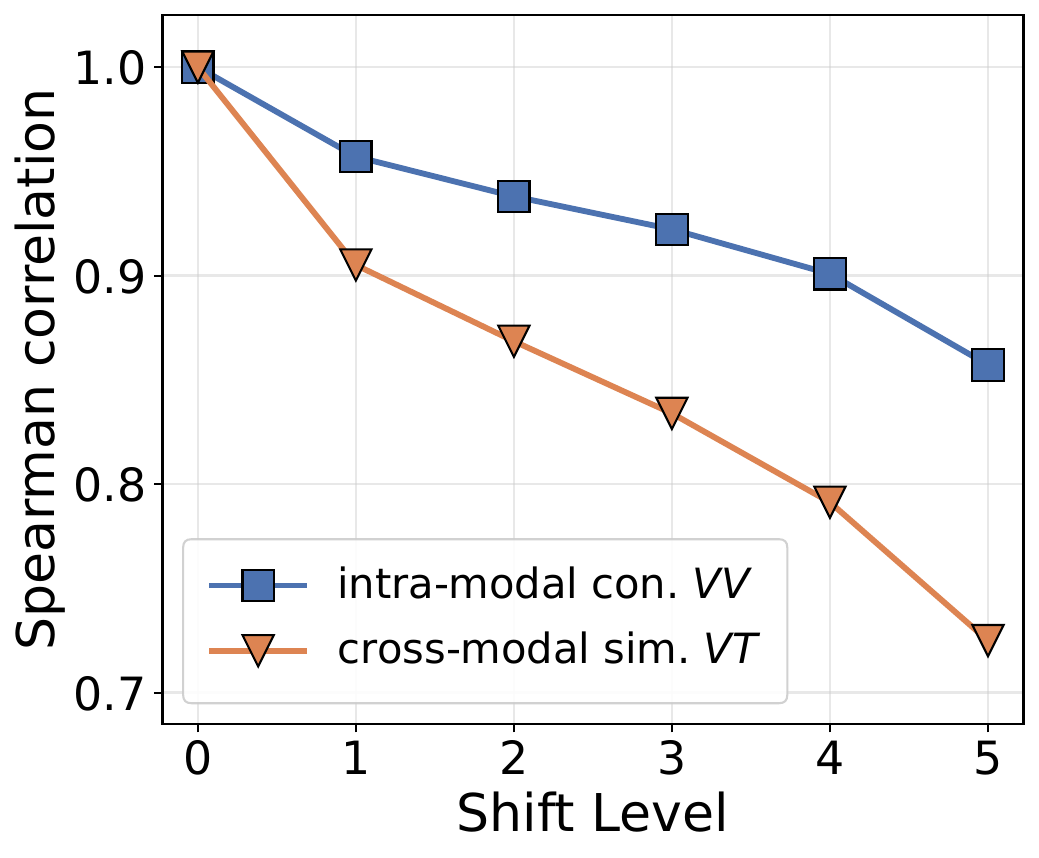}
    \caption{Robustness ($\uparrow$).}
    \label{fig:spearman-panel}
  \end{subfigure}%
  \begin{subfigure}[b]{0.666667\linewidth}
    \centering
    \includegraphics[width=\linewidth]{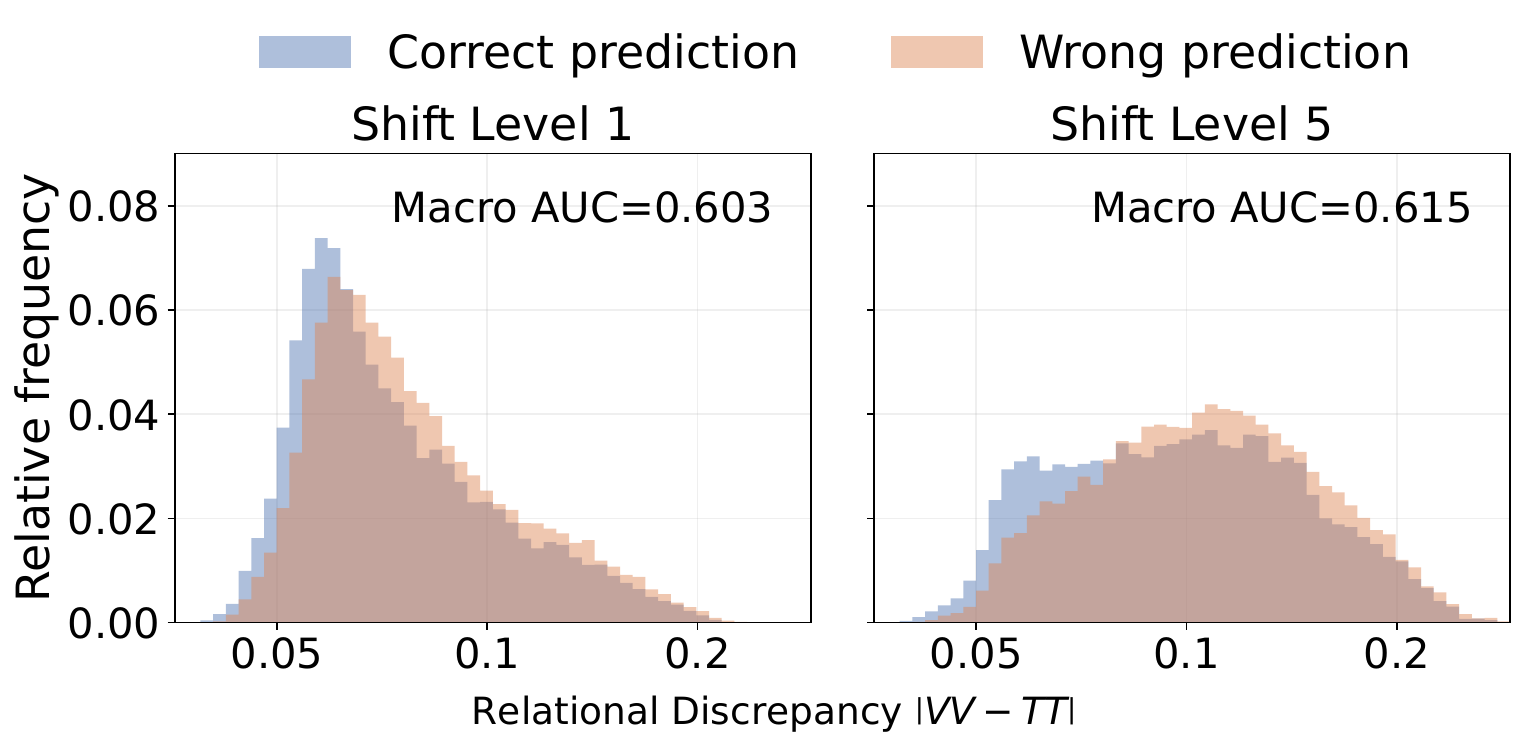}
    \caption{Relational Discrepancy ($\downarrow$).}
    \label{fig:rel-hist}
  \end{subfigure}
  \caption{Observations. \textbf{(1) Robustness.} Intra-modal consistency $\mathbf{V}^\top\mathbf{V}$ remains consistently more robust than cross-modal semantics $\mathbf{V}^\top\mathbf{T}$ in terms of Spearman correlation.
    Increased shift level results in larger robustness gap.
    \textbf{(2) Relational Discrepancy.} Correct prediction exhibits smaller relational discrepancy than wrong prediction. Increased shift level also results in larger relational discrepancy gap.
    More details are provided in Appendix~\ref{sec:exp_setup}.}
  \vspace{-10pt}
  \label{fig:spearman}
\end{figure}

To evaluate our proposed \algname, we conduct extensive experiments on three standard benchmarks across two VLM backbones.
Comprehensive results show that our method outperforms state-of-the-art VLM TTA methods by up to 3\% in mean accuracy while maintaining a superior balance between effectiveness and efficiency.
Our main contributions are summarized as follows:
\begin{itemize}[nosep, wide=0pt, leftmargin=*, after=\strut, label=\textbullet]
  \item \textbf{Novel Relational Perspective.} We approaches VLM TTA from a novel relational perspective, and reveals the critical role of relational structures in improving VLMs under distribution shifts.
  \item \textbf{Unified TTA framework.} We propose a novel VLM TTA method which unifies the inference and adaptation objectives of VLMs at test time via a Fused Gromov-Wasserstein formulation.
  \item \textbf{Strong Performance.} Extensive experiments on three standard VLM TTA benchmarks demonstrate significant performance gain over state-of-the-art baselines with competitive efficiency.
\end{itemize}

\section{Preliminaries}\label{sec:pre}
In this section, we first introduce preliminaries on VLM TTA in Section~\ref{subsec:pre-tta}, followed by Fused Gromov-Wasserstein (FGW) distance in Section~\ref{subsec:pre-ot}.
Throughout the paper, we use bold uppercase letters for matrices (e.g., $\mathbf{S}$), bold lowercase letters for vectors (e.g., $\boldsymbol{\mu}$),
calligraphic uppercase letters for sets (e.g., $\mathcal{T}$), and lowercase letters for scalars (e.g., $k$).

\subsection{Test-time Adaptation of VLMs}\label{subsec:pre-tta}
We denote the image space by $\mathcal{V}$ and text space by $\mathcal{T}$.
A pre-trained VLM, such as CLIP~\citep{radford2021learning},
consists of an image encoder $f_\theta: \mathcal{V}\rightarrow\mathbb{R}^d$ and a text encoder $f_\phi: \mathcal{T}\rightarrow\mathbb{R}^d$,
which map images and texts into a shared embedding space.
For a zero-shot image classification task with $n$ different classes, the image encoder $f_\theta$ projects an image $v_i$ to a visual embedding $\mathbf{v}_i\in\mathbb{R}^d$,
and the text encoder $f_\phi$ projects text descriptions $t_j$ of a class label, e.g., "a photo of a <class>" to textual embeddings $\mathbf{t}_j$.
Predictions are inferred based on the embedding similarity of the test image and text descriptions of class labels, e.g., $\mathop{\arg\max}_j\mathbf{v}^\top\mathbf{t}_j$.
In this paper, we consider the online TTA setting, where test images under distribution shifts arrive sequentially in a batch manner~\citep{maharana2025batclip}.

\subsection{Fused Gromov-Wasserstein Distance}\label{subsec:pre-ot}

The FGW distance serves as a principled distance metric between two distributions across different modalities,
which jointly considers cross-modal semantics and intra-modal relational structures.
Formally, let $\tilde{\mathcal{V}}=\{v_i\}_{i=1}^{n}$ denote a set of $n$ images and $\tilde{\mathcal{T}}=\{t_j\}_{j=1}^{c}$ denote a set of text descriptions of $c$ classes,
$\tilde{\mathcal{V}}$ and $\tilde{\mathcal{T}}$ can be associated with two discrete probability measures $\bs{\mu}_v$ and $\bs{\mu}_t$ as follows:
\begin{equation}
    \bs{\mu}_v=\sum_{i=1}^{n}a_i\delta_{v_i},
    \qquad
    \bs{\mu}_t=\sum_{j=1}^{c}b_j\delta_{t_j}
\end{equation}
where $\mathbf{a}\in\Delta_n$ and $\mathbf{b}\in\Delta_c$ are two histograms representing the image and text weights, respectively. $\delta$ denotes the Dirac function.
In this way, the FGW distance between image set $\tilde{\mathcal{V}}$ and text set $\tilde{\mathcal{T}}$ is defined as follows~\citep{vayer2020fused}.
\begin{definition}
    \textit{Fused Gromov-Wasserstein (FGW) distance} \\
    \textbf{Given:} (1) image set $\tilde{\mathcal{V}}=\{v_i\}_{i=1}^{n}$ and text set $\tilde{\mathcal{T}}=\{t_j\}_{j=1}^{c}$,
    (2) their associated probability measures $\bs{\mu}_v$ and $\bs{\mu}_t$,
    (3) cross-modal cost matrix $\mathbf{M}\in\mathbb{R}^{n\times c}$,
    (4) intra-modal cost matrices $\mathbf{C}^{v}\in\mathbb{R}^{n\times n}$ and $\mathbf{C}^{t}\in\mathbb{R}^{c\times c}$,
    (5) trade-off parameter $\alpha\in[0,1]$,
    (6) the order $q$. \\
    \textbf{Output:} the FGW distance $\mathrm{FGW}_{q,\alpha}(\tilde{\mathcal{V}},\tilde{\mathcal{T}})$
    \begin{equation}\label{eq:fgw-dist}
        \begin{aligned}
            \!\!\!\!\!\mathrm{FGW}_{q,\alpha}(\tilde{\mathcal{V}}, \tilde{\mathcal{T}})
             & =
            \min_{\mathbf{S}\in\Pi(\bs{\mu}_v,\bs{\mu}_t)}
            \;\underbrace{(1-\alpha)
                  \sum_{i=1}^{n}\sum_{j=1}^{c}
                  \mathbf{M}_{ij}^{q}\mathbf{S}_{ij}}_{\text{Wasserstein}}
            \\
             & +\alpha
            \underbrace{\sum_{i,i'=1}^{n}\sum_{j,j'=1}^{c}
                \left|
                \mathbf{C}^{v}_{ii'}-\mathbf{C}^{t}_{jj'}
                \right|^{q}
                \mathbf{S}_{ij}\mathbf{S}_{i'j'}}_{\text{Gromov-Wasserstein}}
        \end{aligned}
    \end{equation}
\end{definition}
Intuitively, the FGW distance linearly combines the Wasserstein distance~\citep{peyre2019computational} and Gromov-Wasserstein distance~\citep{peyre2016gromov}, which capture the cross-modal semantics and intra-modal relational structures, respectively.

\section{Methodology}\label{sec:method}

In this section, we present the proposed \algname.
We first introduce an FGW formulation for robust VLM inference under distribution shifts, in Section~\ref{subsec:me-inf}.
Then, we introduce a multi-level InfoNCE loss for relation-aware VLM adaptation at test time, in Section~\ref{subsec:me-adapt}.
Then, we present the theoretical unification between the FGW formulation and the multi-level InfoNCE loss, in Section~\ref{subsec:me-bridge},
thereby revealing the mutual benefits of the inference and adaptation stages.
Finally, we provide theoretical analysis of \algname\ in Section~\ref{subsec:me-ana}.
An overview of \algname\ is shown in Figure~\ref{fig:model}.

\begin{figure*}[t]
    \includegraphics[width=\linewidth, trim=0 8 0 5, clip]{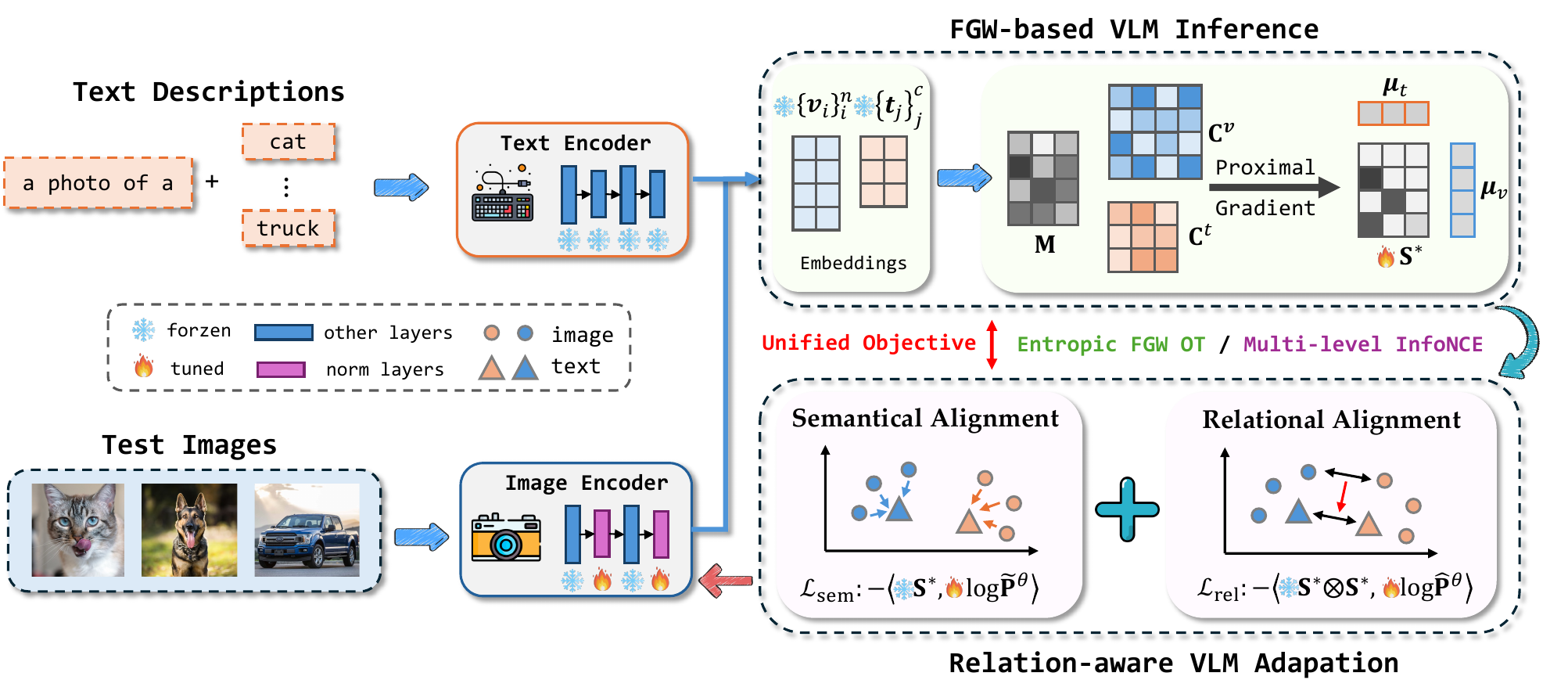}
    \caption{\textbf{An overview of \algname.} VLMs first encode test images via the image encoder and text descriptions of different classes via the text encoder.
        \textbf{During inference}, \algname\ infers zero-shot image classification by solving an FGW formulation via a KL-proximal gradient method.
        \textbf{During adaptation}, \algname\ adapts the normalization layers of the image encoder based on a multi-level InfoNCE loss.}
    \vspace{-10pt}
    \label{fig:model}
\end{figure*}

\subsection{OT-based Inference}\label{subsec:me-inf}
VLMs are pre-trained to project different modalities, i.e., images and texts, into a shared embedding space, making them directly comparable based on the projected embedding similarities.
However, the shared embedding space of VLMs is vulnerable under distribution shifts, making raw embedding noisy to infer the actual semantic similarity between shifted images and texts,
as illustrated in Figure~\ref{fig:teaser}.
Optimal Transport (OT), in contrast, leverages a globally constrained optimization framework to produce noise-reduced similarity across different modalities~\citep{chen2020graph}.
Furthermore, as shown in Figure~\ref{fig:spearman}, while the pairwise embedding distances across two modalities are distorted under distribution shifts,
the intra-modal relational structures remain relatively stable, which can also be captured by a surrogate OT distance metric, i.e., the Gromov-Wasserstein (GW) distance~\citep{peyre2016gromov}.

Therefore, at the inference stage of VLMs, we propose to use the raw image and text embeddings to encode the cross-modal and intra-modal cost matrices,
and instead infer pairwise image-text alignment by solving an FGW formulation.
Formally, given a batch of test images $\tilde{\mathcal{V}}:=\{v_i\}_{i=1}^n$ of size $n$ and text descriptions $\mathcal{T}:=\{t_j\}_{j=1}^c$ of $c$ different text classes,
we predict image-text alignment by solving the FGW formulation in Eq.~\eqref{eq:fgw-dist} with $q=2$, where we define the cross-modal cost matrix $\mathbf{M}_{ij}=\sqrt{1-\mathbf{v}_i^\top\mathbf{t}_j}$
and intra-modal cost matrices $\mathbf{C}^{v}_{ii'}=1-\mathbf{v}_i^\top\mathbf{v}_{i'}, \mathbf{C}^{t}_{jj'}=1-\mathbf{t}_j^\top\mathbf{t}_{j'}$.
We adopt uniform marginals at both the image and text sides, i.e., $\boldsymbol{\mu}_v=\sum_{i=1}^n\frac{1}{n}\delta_{v_i}$ and $\boldsymbol{\mu}_t=\sum_{j=1}^c\frac{1}{c}\delta_{t_j}$.
As we will show in Section~\ref{subsec:me-bridge}, the FGW distance under such costs and marginals is neatly unified with a multi-level InfoNCE loss,
thereby bridging the objective functions of the inference and adaptation stages. The optimal transport plan $\mathbf{S}^{*}\in\mathbb{R}^{n\times c}$ indicates soft alignment between test images and texts at semantic level,
while $\mathbf{S}^{*}\otimes\mathbf{S}^{*}$ indicates their alignment at relational level with $\otimes$ denoting the Kronecker product.

To solve the FGW distance in Eq.~\eqref{eq:fgw-dist}, we adopt a KL-proximal gradient method~\citep{xu2019gromov} where, at the $k$-th iteration,
we linearize the FGW objective around the current transport plan $\mathbf{S}^{(k)}$ and introduce a KL divergence term between two consecutive solutions, yielding the following subproblem:
\begin{equation}
    \mathbf{S}^{(k+1)}
    = \arg\min_{\mathbf{S}\in\Pi(\bs{\mu}_v,\bs{\mu}_t)}
    \;\left\langle \mathbf{L}^{(k)},\mathbf{S}\right\rangle
    +\gamma\,\mathrm{KL}\bigl(\mathbf{S}\|\mathbf{S}^{(k)}\bigr)
\end{equation}
where $\gamma>0$ denotes the proximal regularization weight and $\mathbf{L}^{(k)}$ denotes the gradient of the FGW objective with respect to $\mathbf{S}$ evaluated at $\mathbf{S}^{(k)}$,
i.e., $\mathbf{L}^{(k)}=(1-\alpha)\mathbf{M}^{2}+2\alpha((\mathbf{C}^{v})^{2}\bs{\mu}_v)\mathbf{1}_{c}^{\top}+2\alpha\mathbf{1}_{n}((\mathbf{C}^{t})^{2}\bs{\mu}_t)^{\top}-4\alpha\mathbf{C}^{v}\mathbf{S}^{(k)}(\mathbf{C}^{t})^{\top}$.
The resulting subproblem can be reformulated as an entropic Wasserstein OT problem:
\begin{equation}
    \mathbf{S}^{(k+1)}
    = \arg\min_{\mathbf{S}\in\Pi(\bs{\mu}_v,\bs{\mu}_t)}
    \left\langle \mathbf{C}_{\text{total}}^{(k)},\mathbf{S}\right\rangle
    +\gamma\left\langle\log\mathbf{S},\mathbf{S}\right\rangle
\end{equation}
where $\mathbf{C}_{\text{total}}^{(k)}=\mathbf{L}^{(k)}-\gamma\log\mathbf{S}^{(k)}$. Since $\mathbf{C}_{\text{total}}^{(k)}$ remains fixed within the $k$-th iteration,
the resulting subproblem can be efficiently solved using the Sinkhorn algorithm~\citep{peyre2019computational},
yielding a transport plan that captures correspondences at both semantical and relational levels across different modalities.

Instead of relying on noisy embedding similarities to infer only pairwise semantical correspondences between shifted images and texts,
\algname\ infers multi-level alignment from a globally constrained view by an FGW formulation~\citep{chen2020graph, zhu2025enhancing, yu2025joint},
therefore providing more accurate predictions which produce reliable pseudo-labels for VLM adaptation at both semantical and structural levels.

\vspace{-2pt}
\subsection{Relation-aware Adaptation}\label{subsec:me-adapt}
\vspace{-2pt}
Equipped with the robust pseudo-labels from the inference stage at both the semantical and relational levels,
we propose a multi-level InfoNCE loss to re-align not only pairwise image-text semantics but also their relational structures overlooked by existing VLM TTA methods,
thereby enabling relation-aware adaptation of VLMs at test time.

Formally, for a batch of test images $\mathcal{V}:=\{v_i\}_i^n$ and a set of text descriptions $\mathcal{T}:=\{t_j\}_j^c$ of $c$ different text classes,
we first define a semantic InfoNCE loss $\mathcal{L}_{\text{sem}}$ in Eq.~\eqref{eq:loss-sem} which contrasts the pairwise image-text semantics as follows:
\begin{equation}\label{eq:loss-sem}
    \!\!\mathcal{L}_{\text{sem}}=-\langle\mathbf{S}^*,\log\tilde{\mathbf{P}}^{\theta}\rangle, \quad\!
    \quad \tilde{\mathbf{P}}^{\theta}_{ij}=\frac{\exp(\mathbf{v}_i^\top\mathbf{t}_j/\tau)}{n\sum_{k=1}^c\exp(\mathbf{v}_i^\top\mathbf{t}_k/\tau)}
\end{equation}
where $\mathbf{S}^*$ is the optimal transport plan of OT from the inference stage serving as soft semantic pseudo-labels, and $\tau$ is a scalar temperature.
Comparing Eq.~\eqref{eq:loss-sem} with the pre-training objective of CLIP~\citep{radford2021learning} reveals that the semantic InfoNCE loss $\mathcal{L}_{\text{sem}}$ essentially
\textbf{(1)} replaces the hard, one-hot labels with soft, OT-induced pseudo-labels $\mathbf{S}^*$ at test time,
and \textbf{(2)} aligns shifted image embeddings to text embeddings unidirectionally since distribution shifts happen mainly on the image side~\citep{bao2025mint, maharana2025batclip, zeng2026subspace}.
In this way, the semantic InfoNCE loss aligns with the core philosophy of CLIP pre-training by contrasting the embedding relationships between individual image-text pairs,
thereby enabling fine-grained VLM TTA.

Moreover, to capture the relational structures, we propose a new relational InfoNCE loss $\mathcal{L}_{\text{rel}}$ in Eq.~\eqref{eq:loss-rel}, which contrasts intra-modal structures across two modalities as follows:
\begin{equation}\label{eq:loss-rel}
    \begin{aligned}
        \mathcal{L}_{\text{rel}}
         & =-\langle\mathbf{S}^*\otimes\mathbf{S}^*,\log\hat{\mathbf{P}}^{\theta}\rangle, \quad                                \\ \hat{\mathbf{P}}^{\theta}_{ii',jj'}
         & =\frac{\exp\left(-(\mathbf{v}_i^\top\mathbf{v}_{i'}-\mathbf{t}_j^\top\mathbf{t}_{j'})^2/\tau\right)}
            {n^2\sum_{k,k'=1}^{c}\exp\left(-(\mathbf{v}_i^\top\mathbf{v}_{i'}-\mathbf{t}_k^\top\mathbf{t}_{k'})^2/\tau\right)}
    \end{aligned}
\end{equation}
where $\mathbf{S}^*\otimes\mathbf{S}^*$ serves as soft relational pseudo-labels which indicate relational alignment, i.e., alignment between image-image pair and text-text pair, from the inference stage.
Comparing the relational InfoNCE loss $\mathcal{L}_{\text{rel}}$ in Eq.~\eqref{eq:loss-rel} with the semantic InfoNCE loss $\mathcal{L}_{\text{sem}}$ in Eq.~\eqref{eq:loss-sem} reveals that $\mathcal{L}_{\text{rel}}$
instead re-aligns the relational structures across two modalities by contrasting the pairwise embedding distances of image-image pairs and text-text pairs,
thereby capturing the important intra-modal structures which are overlooked by the semantic InfoNCE loss.

While promising, directly computing the relational InfoNCE loss $\mathcal{L}_{\text{rel}}$ in Eq.~\eqref{eq:loss-rel} requires $O(n^2c^2)$ in terms of both time and space complexity,
making it intractable for large-scale applications.
Therefore, we propose a lightweight approximation of $\mathcal{L}_{\text{rel}}$ by combining a low-rank decomposition of the Gaussian kernel with chunked computation,
as detailed in Appendix~\ref{subsec:low-rank}, reducing its time complexity to $O(r(n^2+c^2))$ and space complexity to $O(nc+nd+cd)$, where $r$ denotes the rank.
Note that such low-rank approximation of $\mathcal{L}_{\text{rel}}$ does not affect the unification of the inference and adaptation objectives in Section~\ref{subsec:me-bridge},
since it is only applied to the computation of $\hat{\mathbf{P}}^{\theta}$ in Eq.~\eqref{eq:loss-rel} which is independent of the transport plan $\mathbf{S}$ at the inference stage.
Detailed complexity analysis of \algname\ is provided in Section~\ref{subsec:me-ana}.

Finally, our proposed \algname\ combines the semantic InfoNCE loss $\mathcal{L}_{\text{sem}}$ and the relational InfoNCE loss $\mathcal{L}_{\text{rel}}$ linearly to form a multi-level InfoNCE loss $\mathcal{L}_{\text{multi}}$ in Eq.~\eqref{eq:loss-adapt} as follows,
enabling relation-aware adaptation of VLMs at test time:
\begin{equation}\label{eq:loss-adapt}
    \mathcal{L}_{\text{multi}}=(1-\beta)\mathcal{L}_{\text{sem}}+\beta\mathcal{L}_{\text{rel}}
\end{equation}
where $\beta \in [0, 1]$ trades off the importance of semantic and relational alignment during adaptation.

\vspace{-2pt}
\subsection{Bridging Inference and Adaptation via OT}\label{subsec:me-bridge}
\vspace{-2pt}

In this section, we further bridge the inference and adaptation objectives of VLMs theoretically through the lens of OT.
Specifically, we identify an interesting mathematical unification between the FGW distance in Eq.~\eqref{eq:fgw-dist} and the multi-level InfoNCE loss in Eq.~\eqref{eq:loss-adapt},
thereby revealing that the inference and adaptation stages of \algname\ optimize the same objective functions w.r.t. different variables.

\begin{lemma}\label{lem:bridge} (Unification of Inference and Adaptation)\\
    The inference and adaptation of \algname\ share the same objective function $\mathcal{T}(\mathbf{S}, \theta)$ w.r.t.
    different optimized variables $\mathbf{S}\in\Pi(\boldsymbol{\mu}_v,\boldsymbol{\mu}_t)$ and $\theta$ of the image encoder as follows:
    \begin{equation}\label{eq:bridge}
        \begin{aligned}
            \mathcal{T}(\mathbf{S}, \theta)
            = & (1-\rho)\underbrace{\sum_{i=1}^n\sum_{j=1}^c -\mathbf{S}_{ij}\log\tilde{\mathbf{P}}^\theta_{ij}}_{\text{Semantic InfoNCE / Wasserstein}} \\
              & +\rho\underbrace{\sum_{i,i'=1}^n\sum_{j,j'=1}^c -\mathbf{S}_{ij}\mathbf{S}_{i'j'}
                         \log\hat{\mathbf{P}}^\theta_{ii',jj'}}_{\text{Relational InfoNCE / Gromov-Wasserstein}}.
        \end{aligned}
    \end{equation}
\end{lemma}

Detailed proof of Lemma~\ref{lem:bridge} can be found in Appendix~\ref{subsec:proof_bridge}. In general, $\mathcal{L}_{\text{sem}}$ in Eq.~\eqref{eq:loss-sem} can be reformulated as a Wasserstein OT term,
while $\mathcal{L}_{\text{rel}}$ in Eq.~\eqref{eq:loss-rel} can be reformulated as a Gromov-Wasserstein OT term, plus constant terms w.r.t. the transport plan $\mathbf{S}$ which do not affect the OT optimization.
Therefore, with $\rho=\alpha=\beta$, both the inference and adaptation stages of \algname\ optimize the same objective $\mathcal{T}(\mathbf{S}, \theta)$ in Eq.~\eqref{eq:bridge}, which linearly combines the semantic InfoNCE loss (Wasserstein distance)
and the relational InfoNCE loss (Gromov-Wasserstein distance) with the same trade-off parameter.

The theoretical unification between the inference and adaptation objectives of \algname\ in Lemma~\ref{lem:bridge} reveals two important insights regarding the gradient consistency and the mutual benefits of the two stages.
Firstly, for \textit{gradient consistency}, by the Envelope theorem~\citep{afriat1971theory}, VLM adaptation using a fixed optimal transport plan $\mathbf{S}^*$ from the inference stage computes the exact gradient of the unified objective $\mathcal{T}(\mathbf{S}, \theta)$,
which guides the learning of the parameters $\theta$ of the image encoder without differentiating through the FGW formulation. Secondly, for \textit{mutual benefits}, the mathematical unification of the inference and adaptation objectives reveals a theoretical
mutual benefit between the two stages beyond simple intuitions: the inference stage provides robust pseudo-labels for effective adaptation, while the adaptation stage learns accurate embeddings for effective inference.
More importantly, as shown in Theorem~\ref{thm:convergence}, alternating between the inference and adaptation stages of \algname\ for the same batch yields non-increasing values of the unified objective $\mathcal{T}(\mathbf{S}, \theta)$,
demonstrating that the two stages jointly push the same objective towards convergence, thereby achieving their mutual benefits from an optimization perspective.
Compared with existing VLM TTA methods that heuristically combine different inference and adaptation techniques with decoupled objectives,
\algname\ provides a principled framework to unleash their theoretical mutual benefits by a unified objective function, thereby achieving effective VLM adaptation.

\vspace{-1pt}
\subsection{Theoretical Analysis}\label{subsec:me-ana}
\vspace{-2pt}
In this subsection, we provide theoretical analysis regarding the convergence and complexity of the proposed \algname.
\vspace{-4pt}
\begin{theorem}\label{thm:convergence} (Convergence of \algname)\\
    Let $\mathbf{S}^{(k)}$ and $\theta^{(k)}$ denote the transport plan and the parameters of the image encoder at the $k$-th iteration,
    the unified objective $\mathcal{T}(\mathbf{S}, \theta)$ in Eq.~\eqref{eq:bridge} is non-increasing w.r.t. $k$, i.e.,
    \begin{equation}
        \mathcal{T}(\mathbf{S}^{(k+1)}, \theta^{(k+1)})\leq\mathcal{T}(\mathbf{S}^{(k)}, \theta^{(k)})
    \end{equation}
\end{theorem}

\vspace{-8pt}
\begin{theorem}\label{thm:complexity} (Complexity of \algname)\\
    The per-batch time complexity of the proposed \algname\ is $O\left(K\left(ncd+\min(n,c)d^2+Jnc\right)+r(n^2+c^2)\right)$, and the per-batch space complexity is $O(nc+nd+cd)$,
    where $n$ is the image batch size, $c$ is the number of text classes, $d$ is the embedding dimension, $K$ is the number of iterations for the KL-proximal gradient method, $J$ is the number of iterations for the Sinkhorn algorithm,
    and $r$ is the rank for low-rank approximation of the relational InfoNCE loss.
\end{theorem}

\vspace{-5pt}
Detailed proofs of Theorem~\ref{thm:convergence} and Theorem~\ref{thm:complexity} can be found in Appendix~\ref{subsec:proof_convergence} and Appendix~\ref{subsec:proof_complexity}, respectively.

\vspace{-2pt}
\section{Experiments}\label{sec:exp}
\vspace{-3pt}

In this section, we carry out extensive experiments and analyses to evaluate \algname\ from the
following aspects:

\begin{itemize}[nosep, wide=0pt, leftmargin=*, after=\strut, label=\textbullet]
    \item \textbf{Q1:} How effective is \algname\ compared to state-of-the-art methods under distribution shifts?
    \item \textbf{Q2:} How does \algname\ trades off effectiveness and efficiency for VLM TTA?
    \item \textbf{Q3:} What are the individual benefits of the inference and adaptation strategies adopted by \algname\ ?
    \item \textbf{Q4:} How sensitive is \algname\ to different hyperparameter and design choices?
\end{itemize}

Due to limited space, we include detailed experimental setup in Appendix~\ref{sec:exp_setup},
including introductions of datasets \& metrics, VLM backbones, TTA baselines, reproducibility details,
and machines used for conducting all experiments. We also include additional experimental results in Appendix~\ref{sec:add_exp_results}.

\begin{table*}[t]
    \centering
    \begingroup
    \definecolor{benchfirst}{HTML}{F5C6C6}
    \definecolor{benchsecond}{HTML}{FAD6AC}
    \definecolor{benchthird}{HTML}{CBE0F3}
    \newcommand{\benchFirst}[1]{\cellcolor{benchfirst}#1}
    \newcommand{\benchSecond}[1]{\cellcolor{benchsecond}#1}
    \newcommand{\benchThird}[1]{\cellcolor{benchthird}#1}
    \setlength{\tabcolsep}{2.5pt}
    \renewcommand{\arraystretch}{1.05}
    \caption{Benchmarking results with ViT-B-16. The 1st/2nd/3rd best results are shaded \colorbox{benchfirst}{\strut red}/\colorbox{benchsecond}{\strut orange}/\colorbox{benchthird}{\strut blue}, respectively.}
    \vspace{-5pt}
    \label{tab:exp_bench_vit_b16}
    \resizebox{\textwidth}{!}{%
        \begin{tabular}{@{}cl*{16}{c}@{}}
            \toprule
            \multicolumn{2}{l}{\multirow{2}{*}{\textbf{Method}}}
                                                                  & \multicolumn{3}{c}{\textbf{Noise}}   & \multicolumn{4}{c}{\textbf{Blur}}
                                                                  & \multicolumn{4}{c}{\textbf{Weather}} & \multicolumn{4}{c}{\textbf{Digital}}
                                                                  & \multirow{2}{*}{\textbf{Mean}}                                                                                                                                                                                                                                                                                                                                                                                        \\
            \cmidrule(lr){3-5}\cmidrule(lr){6-9}\cmidrule(lr){10-13}\cmidrule(lr){14-17}
                                                                  &                                      & \textbf{Gauss.}                      & \textbf{Shot}       & \textbf{Impul.}     & \textbf{Defoc.}     & \textbf{Glass}      & \textbf{Motion}     & \textbf{Zoom}
                                                                  & \textbf{Snow}                        & \textbf{Frost}                       & \textbf{Fog}        & \textbf{Brit.}      & \textbf{Contr.}     & \textbf{Elast.}     & \textbf{Pixel.}     & \textbf{JPEG}       &                                                                                                                                                                                                     \\
            \midrule
            \multirow{12}{*}{\rotatebox{90}{\textbf{CIFAR10-C}}}  & Source                               & 37.68                                & 42.02               & 54.16               & 71.84               & 40.93               & 68.07               & 73.99               & 74.10               & 77.08               & 70.64               & 83.99               & 62.66               & 53.63               & 47.17               & 58.99               & 61.13               \\
                                                                  & ECALP                                & 46.47                                & 50.77               & 61.07               & 71.26               & 42.04               & 70.77               & 74.89               & 75.91               & 79.02               & 71.08               & 86.27               & 60.81               & 57.90               & 49.47               & 61.43               & 63.94               \\
                                                                  & TENT                                 & 15.20                                & 17.97               & 38.16               & \benchThird{81.67}  & 21.28               & 76.32               & 82.19               & 84.00               & 81.79               & 80.36               & 90.24               & 80.20               & 63.44               & 58.48               & 56.50               & 61.85               \\
                                                                  & RoTTA                                & 39.20                                & 43.10               & 55.39               & 71.82               & 41.34               & 68.32               & 73.96               & 74.33               & 77.76               & 70.90               & 85.03               & 63.36               & 54.24               & 49.09               & 60.27               & 61.87               \\
                                                                  & TPT                                  & 38.14                                & 42.40               & 60.66               & 73.32               & 44.86               & 69.46               & 75.80               & 76.13               & 78.51               & 71.90               & 85.58               & 62.43               & 58.92               & 55.24               & 62.46               & 63.72               \\
                                                                  & MEMO                                 & 38.12                                & 41.47               & 55.41               & 72.06               & 41.08               & 68.90               & 73.91               & 75.08               & 77.74               & 71.70               & 84.53               & 61.98               & 55.97               & 49.00               & 61.27               & 61.88               \\
                                                                  & WATT                                 & 46.27                                & 53.29               & 59.86               & 75.22               & 38.29               & 71.92               & 76.19               & 77.30               & 80.73               & 76.11               & 87.86               & 75.73               & 55.12               & 62.43               & 63.18               & 66.63               \\
                                                                  & MINT                                 & 54.26                                & 58.53               & 64.96               & 76.97               & 48.97               & 77.93               & 78.96               & 81.97               & 81.29               & 77.19               & 89.47               & 74.27               & 61.67               & 61.43               & 64.28               & 70.14               \\
                                                                  & BATCLIP                              & \benchThird{62.31}                   & 64.98               & 66.58               & 79.83               & 55.94               & 80.43               & 81.64               & 81.97               & \benchThird{83.91}  & 80.44               & 88.52               & 81.46               & 69.27               & 63.09               & 67.67               & 73.87               \\
                                                                  & SubTTA                               & 59.94                                & \benchThird{66.39}  & \benchThird{69.08}  & 80.57               & \benchThird{59.62}  & \benchThird{80.87}  & \benchThird{82.24}  & \benchThird{84.22}  & 83.26               & \benchThird{82.48}  & \benchThird{90.72}  & \benchSecond{86.34} & \benchThird{70.44}  & \benchThird{73.88}  & \benchThird{69.26}  & \benchThird{75.95}  \\
                                                                  & SAT                                  & \benchSecond{66.13}                  & \benchSecond{68.88} & \benchSecond{74.98} & \benchSecond{81.86} & \benchSecond{65.16} & \benchSecond{81.32} & \benchSecond{82.96} & \benchSecond{84.24} & \benchSecond{84.90} & \benchSecond{83.48} & \benchSecond{90.80} & \benchThird{84.73}  & \benchSecond{73.02} & \benchSecond{76.09} & \benchSecond{70.98} & \benchSecond{77.97} \\
                                                                  & \algname                             & \benchFirst{70.86}                   & \benchFirst{73.60}  & \benchFirst{77.07}  & \benchFirst{83.69}  & \benchFirst{70.08}  & \benchFirst{82.71}  & \benchFirst{85.18}  & \benchFirst{86.45}  & \benchFirst{86.13}  & \benchFirst{86.63}  & \benchFirst{91.56}  & \benchFirst{89.23}  & \benchFirst{76.03}  & \benchFirst{81.45}  & \benchFirst{74.58}  & \benchFirst{81.02}  \\
            \midrule
            \multirow{12}{*}{\rotatebox{90}{\textbf{CIFAR100-C}}} & Source                               & 19.81                                & 21.28               & 24.89               & 42.72               & 20.33               & 43.54               & 48.31               & 48.44               & 49.67               & 41.41               & 56.71               & 34.49               & 29.11               & 23.80               & 32.44               & 35.80               \\
                                                                  & ECALP                                & 23.41                                & 25.33               & 30.47               & 43.02               & 19.47               & 42.84               & 49.62               & 47.78               & 50.27               & 42.73               & 57.97               & 35.14               & 29.99               & 25.44               & 31.88               & 37.02               \\
                                                                  & TENT                                 & 7.36                                 & 7.66                & 8.79                & \benchThird{51.84}  & 8.33                & \benchThird{52.24}  & 53.81               & 52.54               & 36.73               & 48.03               & 63.09               & \benchThird{52.93}  & \benchThird{36.51}  & 39.46               & 38.16               & 37.17               \\
                                                                  & RoTTA                                & 20.91                                & 22.60               & 26.50               & 42.84               & 20.40               & 43.31               & 48.20               & 49.08               & 49.80               & 41.52               & 57.74               & 34.53               & 28.81               & 25.53               & 33.28               & 36.34               \\
                                                                  & TPT                                  & 18.31                                & 19.52               & 27.61               & 43.60               & 19.59               & 42.31               & 48.16               & 48.77               & 49.08               & 41.99               & 57.59               & 33.11               & 31.21               & 27.26               & 32.86               & 36.06               \\
                                                                  & MEMO                                 & 19.84                                & 22.40               & 27.58               & 43.56               & 20.04               & 44.43               & 49.06               & 50.51               & 50.02               & 43.39               & 58.91               & 34.57               & 30.44               & 25.38               & 33.73               & 36.92               \\
                                                                  & WATT                                 & 25.83                                & 27.29               & 31.80               & 48.38               & 23.37               & 48.87               & 52.72               & 53.28               & 52.18               & 48.12               & 62.86               & 45.60               & 35.41               & 37.00               & 38.06               & 42.05               \\
                                                                  & MINT                                 & 27.86                                & 30.39               & 35.77               & 48.91               & 26.09               & 47.26               & 53.26               & 52.62               & 52.54               & \benchThird{48.39}  & \benchSecond{64.42} & 43.90               & 35.18               & 32.56               & 35.17               & 42.29               \\
                                                                  & BATCLIP                              & 25.52                                & 28.41               & 34.17               & 50.04               & 26.57               & 48.67               & \benchThird{55.23}  & 51.92               & 51.46               & 48.29               & 62.29               & 45.86               & 34.39               & 33.03               & 36.78               & 42.18               \\
                                                                  & SubTTA                               & \benchThird{31.35}                   & \benchThird{34.68}  & \benchThird{35.80}  & 49.73               & \benchThird{29.23}  & 47.97               & 53.71               & \benchThird{54.04}  & \benchThird{53.22}  & 48.20               & 62.78               & 46.30               & 35.56               & \benchThird{40.57}  & \benchSecond{41.06} & \benchThird{44.28}  \\
                                                                  & SAT                                  & \benchSecond{36.31}                  & \benchSecond{38.21} & \benchSecond{43.87} & \benchSecond{52.97} & \benchSecond{34.51} & \benchSecond{52.61} & \benchSecond{56.49} & \benchSecond{55.50} & \benchSecond{55.49} & \benchSecond{52.91} & \benchFirst{65.42}  & \benchSecond{54.70} & \benchSecond{41.00} & \benchSecond{45.44} & \benchThird{41.01}  & \benchSecond{48.43} \\
                                                                  & \algname                             & \benchFirst{39.12}                   & \benchFirst{41.06}  & \benchFirst{45.81}  & \benchFirst{53.17}  & \benchFirst{39.35}  & \benchFirst{53.20}  & \benchFirst{57.46}  & \benchFirst{55.99}  & \benchFirst{56.17}  & \benchFirst{55.54}  & \benchThird{63.46}  & \benchFirst{60.04}  & \benchFirst{45.16}  & \benchFirst{49.39}  & \benchFirst{43.83}  & \benchFirst{50.58}  \\
            \midrule
            \multirow{12}{*}{\rotatebox{90}{\textbf{ImageNet-C}}} & Source                               & 11.41                                & 12.87               & 11.87               & 23.27               & 15.40               & 24.89               & 22.17               & 31.97               & 29.91               & 35.58               & 54.37               & 17.33               & 12.72               & 30.88               & 33.50               & 24.54               \\
                                                                  & ECALP                                & 13.31                                & 15.20               & 13.60               & 22.10               & 15.12               & 25.63               & 23.38               & 30.83               & 30.18               & 35.61               & 51.91               & 18.54               & 13.12               & 33.62               & 33.61               & 25.05               \\
                                                                  & TENT                                 & 5.48                                 & 5.32                & 7.94                & 25.51               & 19.48               & 26.76               & 24.28               & 33.82               & 30.77               & 38.27               & 54.62               & 22.56               & 13.88               & 34.76               & 35.63               & 25.27               \\
                                                                  & RoTTA                                & 10.93                                & 12.71               & 12.51               & 23.59               & 15.84               & 24.88               & 22.44               & 32.09               & 29.64               & 36.23               & 53.92               & 17.23               & 12.63               & 31.12               & 33.14               & 24.59               \\
                                                                  & TPT                                  & 8.09                                 & 9.37                & 10.36               & 23.61               & 15.24               & 24.89               & 24.43               & 34.42               & \benchThird{32.28}  & 36.92               & 55.12               & 16.10               & 14.21               & 34.08               & 33.78               & 24.86               \\
                                                                  & MEMO                                 & 10.64                                & 12.02               & 11.92               & 23.68               & 15.53               & 24.68               & 23.20               & 32.88               & 29.80               & 35.87               & 54.32               & 17.32               & 12.27               & 31.62               & 32.98               & 24.58               \\
                                                                  & WATT                                 & 11.37                                & 13.11               & 13.31               & 25.61               & 18.37               & 26.78               & 25.04               & 33.44               & 29.84               & 37.91               & 53.94               & 21.38               & 15.44               & 33.19               & 34.79               & 26.23               \\
                                                                  & MINT                                 & \benchThird{19.72}                   & 19.88               & 19.24               & \benchSecond{26.52} & 21.61               & 29.63               & 25.64               & 32.30               & 29.28               & 39.30               & \benchThird{55.32}  & 24.04               & 18.83               & 36.21               & \benchSecond{38.19} & 29.05               \\
                                                                  & BATCLIP                              & 19.67                                & \benchThird{20.91}  & \benchThird{19.43}  & \benchThird{26.08}  & 21.10               & \benchThird{30.44}  & \benchThird{28.80}  & \benchThird{35.38}  & 31.16               & \benchThird{40.52}  & \benchSecond{55.62} & \benchThird{26.04}  & 24.06               & \benchSecond{36.92} & 37.41               & \benchThird{30.24}  \\
                                                                  & SubTTA                               & \benchSecond{21.10}                  & \benchSecond{21.89} & \benchSecond{20.73} & 24.47               & \benchSecond{23.94} & \benchSecond{30.85} & \benchSecond{29.50} & \benchSecond{37.46} & \benchFirst{35.44}  & \benchSecond{41.74} & \benchFirst{56.91}  & \benchSecond{29.90} & \benchSecond{29.28} & \benchThird{36.72}  & \benchThird{38.02}  & \benchSecond{31.86} \\
                                                                  & SAT                                  & 18.66                                & 20.60               & 19.34               & 21.32               & \benchThird{21.90}  & 24.40               & 25.48               & 29.38               & 32.04               & 31.06               & 45.64               & 25.28               & \benchThird{24.60}  & 25.92               & 26.60               & 26.15               \\
                                                                  & \algname                             & \benchFirst{21.84}                   & \benchFirst{23.16}  & \benchFirst{22.46}  & \benchFirst{27.74}  & \benchFirst{27.34}  & \benchFirst{32.14}  & \benchFirst{30.88}  & \benchFirst{38.66}  & \benchSecond{35.14} & \benchFirst{41.94}  & 53.86               & \benchFirst{33.20}  & \benchFirst{31.40}  & \benchFirst{40.68}  & \benchFirst{40.02}  & \benchFirst{33.36}  \\
            \bottomrule
        \end{tabular}%
    }
    \endgroup
    \vspace{-10pt}
\end{table*}

\vspace{-2pt}
\subsection{Benchmarking Results}
\vspace{-2pt}

We first benchmark \algname\ with ViT-B-16 visual encoder on three standard VLM TTA benchmarks, whose results are reported in Table~\ref{tab:exp_bench_vit_b16}.
We observe that \textbf{(1)} \textbf{The proposed \algname\ consistently achieves state-of-the-art performance across all three benchmarks.}
Specifically, on datasets with low-resolution images, \algname\ outperforms the best competitor by a large margin, with a 3.1\% and 1.9\% improvement in mean accuracy on CIFAR10-C and CIFAR100-C, respectively.
For the larger-scale ImageNet-C dataset, \algname\ also achieves the best mean accuracy with at least 1.5\% performance improvement, demonstrating the effectiveness of \algname.
\textbf{(2)} \textbf{OT inherently produces accurate predictions without adapting VLM parameters.}
Comparing the performance of OT inference \textit{without} adaptation in Table~\ref{tab:exp_ablation} with that of training-free baselines, i.e., TENT and ECALP,
shows that OT inference alone outperforms both training-free heuristics in mean accuracy on CIFAR100-C,
revealing that OT itself can infer accurate predictions, thus generating robust pseudo-labels that lead to effective adaptation.
\textbf{(3) The multi-level InfoNCE loss effectively restores both cross-modal semantic and relational correspondences.}
As shown in Figure~\ref{fig:exp_tsne}, compared to other baselines,
\algname\ produces image embeddings that not only align closely with the text embeddings of the same class,
but also show clear clustering structures in the embedding space.
This demonstrates that the multi-level InfoNCE loss largely recovers both the semantic and relational correspondences across different modalities under distribution shifts,
thereby achieving superior TTA performance.
Additional benchmarking results with ViT-B-32 are provided in Table~\ref{tab:exp_bench_vit_b32} in the Appendix.

\begin{figure*}[t]
    \includegraphics[width=\linewidth, trim=0 3 0 3, clip]{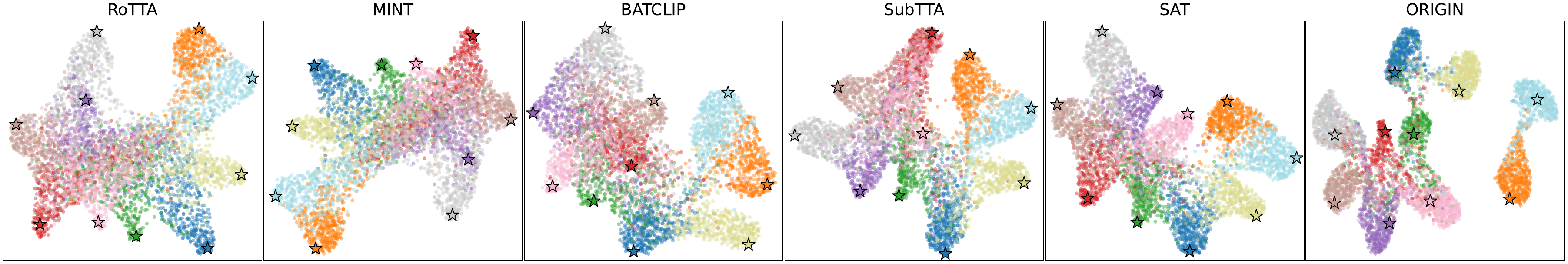}
    \vspace{-15pt}
    \caption{Embedding visualization of different VLM TTA methods by t-SNE on CIFAR10-C. Stars denote text embeddings of textual descriptions of different classes, and round points denote image embeddings.}
    \vspace{-12pt}
    \label{fig:exp_tsne}
\end{figure*}

\begin{figure}[t]
    \vspace{0pt}
    \centering
    \includegraphics[width=\linewidth, trim=0 8 0 5, clip]{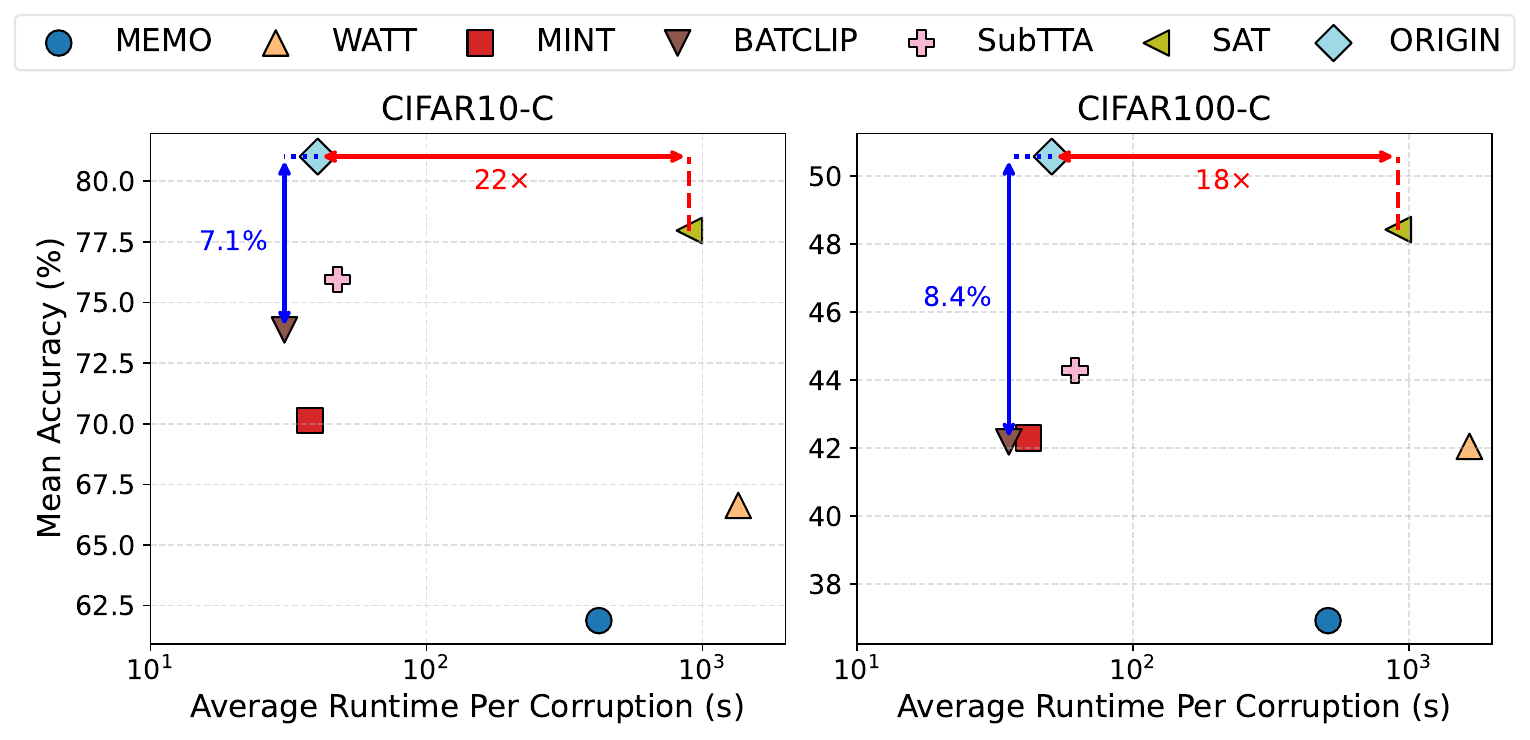}
    \vspace{-15pt}
    \caption{Accuracy vs. runtime trade-off of \algname\ and other training-based TTA methods,
        showing that the proposed \algname\ achieves a superior balance between effectiveness and efficiency.}
    \label{fig:exp_acc_vs_eff}
    \vspace{-10pt}
\end{figure}

\vspace{-30pt}

\subsection{Efficiency Results}
We study the effectiveness-efficiency trade-off of the proposed \algname\ and other training-based TTA methods,
including MEMO, WATT, MINT, BATCLIP, SubTTA, and SAT on CIFAR10-C and CIFAR100-C.
The results on mean accuracy vs. average runtime per corruption,
defined as the averaged wall-clock runtime required for inference and adaptation over different corruption types, are shown in Figure~\ref{fig:exp_acc_vs_eff}.
We can see that \textbf{\algname\ achieves a superior balance between effectiveness and efficiency},
with the corresponding data points lying consistently on the upper-left region of both scatter plots.
Specifically, \algname\ outperforms the fastest baseline BATCLIP with up to 8.4\% improvement in mean accuracy while introducing minimal additional latency.
This demonstrates that the computation of FGW distance and the relational InfoNCE loss incurs relatively small overheads compared to the adaptation of VLM parameters.
Compared with the best-performing baseline SAT, \algname\ runs significantly faster with up to 22$\times$ speed-up in runtime.
While both methods use OT, SAT samples multiple text templates for repeated adaptation per-batch to improve robustness,
whereas \algname\ only requires a single adaptation step per-batch and instead leverages the relational structures across modalities for robust TTA,
thus achieving significantly better efficiency.

\begin{table}[t]
    \vspace{-2pt}
    \caption{Accuracy (\%) of different variants of \algname\ on CIFAR100-C. \algname\ is equivalent to \textsc{FGW}+$\mathcal{L}_\text{multi}$.}
    \vspace{-5pt}
    \centering
    \setlength\tabcolsep{5pt}
    \label{tab:exp_ablation}
    \resizebox{0.95\linewidth}{!}{
        \begin{tabular}{lccccc}
            \toprule
            Dataset                               & \multicolumn{5}{c}{CIFAR100-C}                                     \\
            \cmidrule(lr){1-1} \cmidrule(lr){2-6}
            Corruption                            & Noise                          & Blur  & Weather & Digital & Mean  \\
            \midrule

            \textsc{EMB}                          & 22.08                          & 38.38 & 49.23   & 30.02   & 35.78 \\
            \textsc{W}                            & 28.95                          & 40.20 & 49.20   & 33.15   & 38.47 \\
            \textsc{FGW}                          & 29.12                          & 40.65 & 50.38   & 33.10   & 38.93 \\
            \midrule
            \textsc{EMB}+$\mathcal{L}_\text{sem}$ & 2.23                           & 19.26 & 28.35   & 18.38   & 18.04 \\
            \textsc{W}+$\mathcal{L}_\text{sem}$   & 40.38                          & 47.64 & 53.43   & 46.46   & 47.42 \\
            \midrule
            \textsc{FGW}+$\mathcal{L}_\text{sem}$ & 40.95                          & 49.80 & 56.37   & 48.75   & 49.50 \\
            \textsc{W}+$\mathcal{L}_\text{multi}$ & 40.94                          & 48.14 & 54.50   & 47.38   & 48.19 \\
            \algname\                             & 42.00                          & 50.80 & 57.79   & 49.60   & 50.58 \\

            \bottomrule
        \end{tabular}
    }
    \vspace{-10pt}
\end{table}

\subsection{Further Studies}

\subsubsection{Ablation Study}

To isolate the individual effect of the FGW formulation and multi-level InfoNCE loss at the inference and adaptation stages, respectively,
we compare the performance of \algname\ against different combinations of inference and adaptation strategies based on OT and InfoNCE loss.
Specifically, for inference strategies, we consider the following 3 variants:
\begin{itemize}[nosep, wide=0pt, leftmargin=*, after=\strut, label=\textbullet]
    \item \textsc{Emb}: embedding-based inference via cosine similarity.
    \item \textsc{W}: OT-based inference based on Wasserstein distance.
    \item \textsc{FGW}: OT-based inference based on FGW distance.
\end{itemize}
For adaptation strategies, we consider the following 2 loss functions:
\begin{itemize}[nosep, wide=0pt, leftmargin=*, after=\strut, label=\textbullet]
    \item $\mathcal{L}_\text{sem}$: Semantic InfoNCE in Eq.~\eqref{eq:loss-sem}
    \item $\mathcal{L}_\text{multi}$: Multi-level InfoNCE in Eq.~\eqref{eq:loss-adapt}
\end{itemize}

The results of their combinations are shown in Table~\ref{tab:exp_ablation}.
\algname\ is equivalent to \textsc{FGW}+$\mathcal{L}_\text{multi}$.
The first four rows of Table~\ref{tab:exp_ablation} show the performance of different inference strategies under distribution shift \textit{without adaptation}.
We observe that:
\textbf{(1) FGW-based inference predicts the most reliable pseudo-labels under distribution shift.}
Specifically, by comparing the first three rows of Table~\ref{tab:exp_ablation} which measure the accuracy of different inference strategies under distribution shifts,
we can see that while OT-based inference consistently outperforms embedding-based inference,
the FGW formulation further improves over the Wasserstein formulation by exploiting the intra-modal relational structures of image and text embeddings,
thereby producing the most reliable pseudo-labels for adaptation under distribution shifts.
\textbf{(2) Noisy pseudo-labels lead to degraded performance.}
By comparing the middle two rows of Table~\ref{tab:exp_ablation},
we can see that the performance gap of OT-based inference over embedding-based inference increases sharply after adaptation with a 29\% outperformance.
Notably, the performance of embedding-based inference degrades after adaptation due to excessive noise amplification from unreliable pseudo-labels,
which demonstrates the importance of a reliable inference strategy for effective adaptation under distribution shifts.
\textbf{(3) Relational InfoNCE loss improves the adaptation performance.}
Variants of \algname\ with multi-level InfoNCE loss consistently outperform their counterparts with semantic InfoNCE loss only,
i.e., \textsc{W}+$\mathcal{L}_\text{multi}$ vs. \textsc{W}+$\mathcal{L}_\text{sem}$ and \textsc{FGW}+$\mathcal{L}_\text{multi}$ vs. \textsc{FGW}+$\mathcal{L}_\text{sem}$,
which demonstrates the effectiveness of relational InfoNCE loss in improving the adaptation of VLMs under distribution shifts.

\begin{figure}[t]
    \includegraphics[width=\linewidth, trim=0 7 0 7, clip]{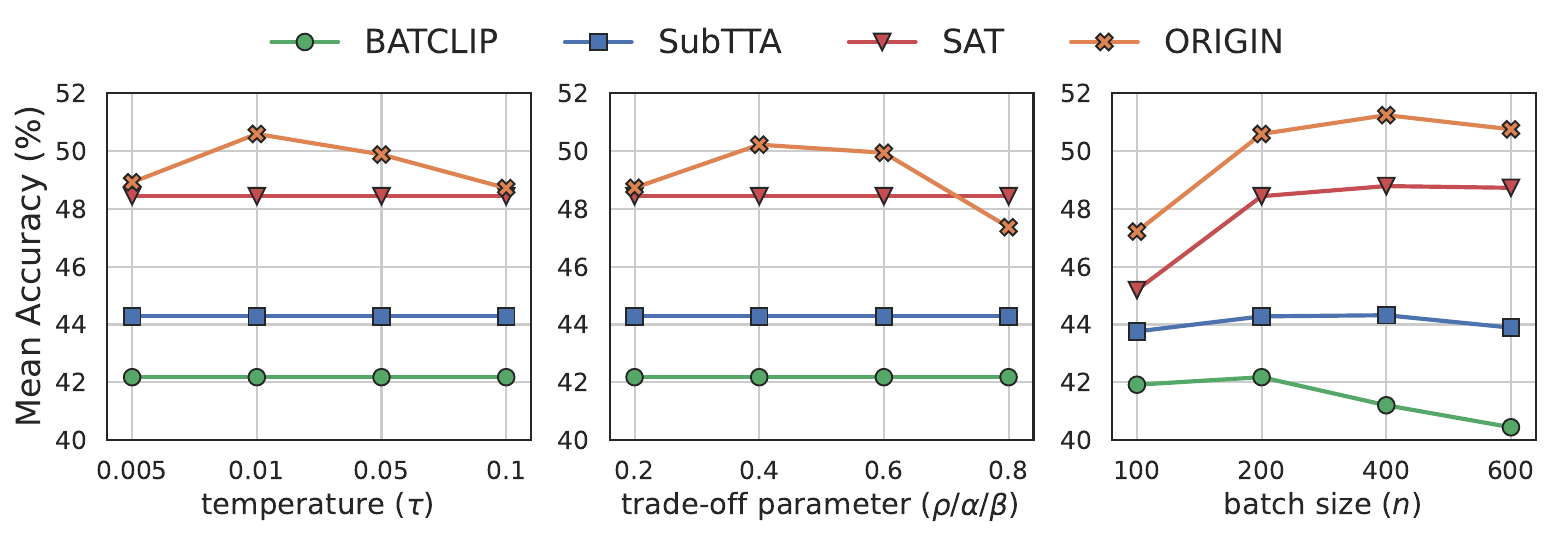}
    \vspace{-10pt}
    \caption{Hyperparameter sensitivity study of \algname\ on CIFAR100-C}
    \vspace{-10pt}
    \label{fig:exp_param}
    \vspace{-5pt}
\end{figure}

\subsubsection{Hyperparameter Analysis}

We study the sensitivity of \algname\ to the following three key hyperparameters: the temperature $\tau$ in Eq.~\eqref{eq:loss-adapt},
the trade-off parameter $\rho$ in Eq.~\eqref{eq:bridge}, and the batch size $n$.
\textbf{In general, the performance of \algname\ remains relatively stable under different choices of hyperparameters with consistent outperformance than its best competitors}.
Detailed analysis for each parameter is provided as follows.
\vspace{-5pt}
\paragraph{Temperature $\tau$.} We find that overly small or large $\tau$ typically leads to a slight performance degradation for \algname.
A small $\tau$ sharpens the softmax distribution but could potentially introduce errors due to overconfidence.
A large $\tau$ avoids overly confident prediction, but leads to weaker contrastive signals which limit the adaptation performance.
\vspace{-5pt}
\paragraph{Trade-off parameter $\rho$.} We can see that while \algname\ generally benefits from a $\rho>0$ by taking relational structures into account,
overly large $\rho$ leads to a noticeable performance degradation.
This indicates that direct alignment between image-text pairs that align with the pretraining objective remains crucial to the adaptation performance.
\vspace{-5pt}
\paragraph{Batch size $n$.} The influence of batch size $n$ shares similar features as that of $\tau$, where overly small or large $n$ leads to performance degradation.
Small $n$ allows finer-grained adaptation but provides limited global information for FGW-based inference per batch.
By contrast, large $n$ leads to more accurate FGW-based inference on a batch level but allows limited steps for adaptation.

\section{Related Works}\label{sec:related}
\paragraph{VLM Test-time Adaptation.}
Early TTA methods for VLMs focus primarily on prompt optimization~\citep{shu2022tpt, xiao2025dynaprompt},
memory management for historical samples~\citep{karmanov2024efficient, zhang2024dual, li2024efficient},
or data augmentation \& selection ~\citep{dobler2024lost},
without adapting parameters inside different encoders of VLMs.
More recent training-based methods attempt to adapt normalization layers of the encoders inside VLMs by re-aligning image and text embeddings under distribution shifts.
Specifically, BATCLIP~\citep{maharana2025batclip} adapts both the image and text encoders of CLIP by simultaneously separating image prototypes and aligning image-text prototypes.
SubTTA~\citep{zeng2026subspace} aligns visual-textual subspaces and filters visual nuisance via semantic projection for robust VLM TTA.
SAT~\citep{zeng2026subspace} leverages a multi-template distillation approach for VLM TTA based on OT-inferred pseudo-labels.
While promising, none of them explicitly leverage or preserve the relational structures across two modalities during TTA.

\paragraph{Optimal Transport for VLMs.}
OT has attracted increasing attention in multimodal learning and vision-language models.
PLOT~\citep{chen2022plot} and Prompt-OT~\citep{chen2026prompt} leverage OT for prompt learning to improve VLMs on downstream vision-language tasks.
OT-CLIP~\citep{shi2024ot} leverages different variants of OT to improve the pretraining performance of CLIP.
GeoTTER~\citep{lee2026geotter} improves the local geometry of OT for zero-shot image classification by leveraging target neighborhood and cluster structure.
\citep{li2026rethinking} identifies the critical role of relational consistency across modalities for vision-encoder selection in VLMs measured by Gromov-Wasserstein distance.
While OT has been widely adopted for VLMs, none of them jointly consider both cross-modal semantics and intra-modalrelations via OT for robust and effective VLM TTA.

\section{Conclusion}\label{sec:con}
In this paper, we propose a unified VLM TTA method \algname\ which jointly models the cross-modal semantics and intra-modal relational structures via Fused Gromov-Wasserstein (FGW) alignment.
For inference, we formulate the zero-shot image classification task as a cross-modal alignment problem via an FGW formulation,
leveraging both pairwise image-text semantics and their relational consistency to generate robust pseudo-labels under distribution shifts.
For adaptation, we propose a multi-level InfoNCE loss to re-align both semantical correspondence and relational structures across the two modalities for effective VLM adaptation.
Finally, we bridge the inference and adaptation objectives of \algname\ by unifying the FGW distance and the multi-level InfoNCE loss,
thus revealing their mutual benefits from a theoretical perspective.
Extensive experiments show consistent outperformance of \algname\ against state-of-the-art VLM TTA methods with competitive efficiency.

\clearpage
\section*{LLM Usage Statement}
LLMs are used only to assist in polishing the grammar and writing of this paper.
We have reviewed all AI-assisted work, and take full responsibility for the final content of this work, including text, claims and artifacts.

{\small
    \bibliographystyle{./ieeenat_fullname}
    \bibliography{sections/reference}

@inproceedings{radford2021learning,
  title        = {Learning transferable visual models from natural language supervision},
  author       = {Radford, Alec and Kim, Jong Wook and Hallacy, Chris and Ramesh, Aditya and Goh, Gabriel and Agarwal, Sandhini and Sastry, Girish and Askell, Amanda and Mishkin, Pamela and Clark, Jack and others},
  booktitle    = {International conference on machine learning},
  pages        = {8748--8763},
  year         = {2021},
  organization = {PmLR}
}

@inproceedings{li2022blip,
  title     = {BLIP: Bootstrapping Language-Image Pre-training for Unified Vision-Language Understanding and Generation},
  author    = {Li, Junnan and Li, Dongxu and Xiong, Caiming and Hoi, Steven},
  booktitle = {Proceedings of the 39th International Conference on Machine Learning},
  pages     = {12888--12900},
  year      = {2022},
  publisher = {PMLR}
}

@inproceedings{luddecke2022image,
  title     = {Image Segmentation Using Text and Image Prompts},
  author    = {L{\"u}ddecke, Timo and Ecker, Alexander},
  booktitle = {Proceedings of the IEEE/CVF Conference on Computer Vision and Pattern Recognition},
  pages     = {7086--7096},
  year      = {2022}
}

@article{hendrycks2019benchmarking,
  title   = {Benchmarking neural network robustness to common corruptions and perturbations},
  author  = {Hendrycks, Dan and Dietterich, Thomas},
  journal = {arXiv preprint arXiv:1903.12261},
  year    = {2019}
}

@inproceedings{shu2022tpt,
  title     = {Test-time prompt tuning for zero-shot generalization in vision-language models},
  author    = {Shu, Manli and Nie, Weili and Huang, De-An and Chen, Zhiding and Goldstein, Tom and Anandkumar, Anima and Xiao, Chaowei},
  booktitle = {NeurIPS},
  year      = {2022}
}

@inproceedings{li2024tda,
  title     = {Test-time Domain Adaptation with Positive and Negative Caches},
  author    = {Li, Mingxuan and others},
  booktitle = {CVPR},
  year      = {2024}
}

@inproceedings{sheng2025illusion,
  title     = {The Illusion of Progress? A Critical Look at Test-Time Adaptation for Vision-Language Models},
  author    = {Sheng, Lijun and Liang, Jian and He, Ran and Wang, Zilei and Tan, Tieniu},
  booktitle = {Proc. NeurIPS},
  year      = {2025}
}

@article{wang2020tent,
  title   = {Tent: Fully test-time adaptation by entropy minimization},
  author  = {Wang, Dequan and Shelhamer, Evan and Liu, Shaoteng and Olshausen, Bruno and Darrell, Trevor},
  journal = {arXiv preprint arXiv:2006.10726},
  year    = {2020}
}

@inproceedings{yu2025joint,
  title     = {Joint optimal transport and embedding for network alignment},
  author    = {Yu, Qi and Zeng, Zhichen and Yan, Yuchen and Ying, Lei and Srikant, R and Tong, Hanghang},
  booktitle = {Proceedings of the ACM on Web Conference 2025},
  pages     = {2064--2075},
  year      = {2025}
}

@inproceedings{maharana2025batclip,
  title     = {Batclip: Bimodal online test-time adaptation for clip},
  author    = {Maharana, Sarthak and Zhang, Baoming and Karlinsky, Leonid and Feris, Rogerio and Guo, Yunhui},
  booktitle = {Proceedings of the IEEE/CVF International Conference on Computer Vision},
  pages     = {1569--1579},
  year      = {2025}
}

@book{peyre2019computational,
  title     = {Computational optimal transport: With applications to data science},
  author    = {Peyr{\'e}, Gabriel and Cuturi, Marco},
  year      = {2019},
  publisher = {Now Foundations and Trends}
}

@article{krizhevsky2009learning,
  title     = {Learning multiple layers of features from tiny images},
  author    = {Krizhevsky, Alex and Hinton, Geoffrey and others},
  year      = {2009},
  publisher = {Toronto, ON, Canada}
}

@inproceedings{deng2009imagenet,
  title        = {Imagenet: A large-scale hierarchical image database},
  author       = {Deng, Jia and Dong, Wei and Socher, Richard and Li, Li-Jia and Li, Kai and Fei-Fei, Li},
  booktitle    = {2009 IEEE conference on computer vision and pattern recognition},
  pages        = {248--255},
  year         = {2009},
  organization = {Ieee}
}

@inproceedings{karmanov2024efficient,
  title     = {Efficient test-time adaptation of vision-language models},
  author    = {Karmanov, Adilbek and Guan, Dayan and Lu, Shijian and El Saddik, Abdulmotaleb and Xing, Eric},
  booktitle = {Proceedings of the IEEE/CVF Conference on Computer Vision and Pattern Recognition},
  pages     = {14162--14171},
  year      = {2024}
}

@inproceedings{zhang2024dual,
  title     = {Dual memory networks: A versatile adaptation approach for vision-language models},
  author    = {Zhang, Yabin and Zhu, Wenjie and Tang, Hui and Ma, Zhiyuan and Zhou, Kaiyang and Zhang, Lei},
  booktitle = {Proceedings of the IEEE/CVF conference on computer vision and pattern recognition},
  pages     = {28718--28728},
  year      = {2024}
}

@inproceedings{dobler2024lost,
  title        = {A lost opportunity for vision-language models: a comparative study of online test-time adaptation for vision-language models},
  author       = {D{\"o}bler, Mario and Marsden, Robert A and Raichle, Tobias and Yang, Bin},
  booktitle    = {European Conference on Computer Vision},
  pages        = {117--133},
  year         = {2024},
  organization = {Springer}
}

@article{farina2024frustratingly,
  title   = {Frustratingly easy test-time adaptation of vision-language models},
  author  = {Farina, Matteo and Franchi, Gianni and Iacca, Giovanni and Mancini, Massimiliano and Ricci, Elisa},
  journal = {Advances in Neural Information Processing Systems},
  volume  = {37},
  pages   = {129062--129093},
  year    = {2024}
}

@article{li2024efficient,
  title   = {Efficient and context-aware label propagation for zero-/few-shot training-free adaptation of vision-language model},
  author  = {Li, Yushu and Su, Yongyi and Goodge, Adam and Jia, Kui and Xu, Xun},
  journal = {arXiv preprint arXiv:2412.18303},
  year    = {2024}
}

@inproceedings{yuan2023robust,
  title     = {Robust test-time adaptation in dynamic scenarios},
  author    = {Yuan, Longhui and Xie, Binhui and Li, Shuang},
  booktitle = {Proceedings of the IEEE/CVF Conference on Computer Vision and Pattern Recognition},
  pages     = {15922--15932},
  year      = {2023}
}

@article{shu2022test,
  title   = {Test-time prompt tuning for zero-shot generalization in vision-language models},
  author  = {Shu, Manli and Nie, Weili and Huang, De-An and Yu, Zhiding and Goldstein, Tom and Anandkumar, Anima and Xiao, Chaowei},
  journal = {Advances in Neural Information Processing Systems},
  volume  = {35},
  pages   = {14274--14289},
  year    = {2022}
}

@article{zhang2022memo,
  title   = {Memo: Test time robustness via adaptation and augmentation},
  author  = {Zhang, Marvin and Levine, Sergey and Finn, Chelsea},
  journal = {Advances in neural information processing systems},
  volume  = {35},
  pages   = {38629--38642},
  year    = {2022}
}

@article{osowiechi2024watt,
  title   = {Watt: Weight average test time adaptation of clip},
  author  = {Osowiechi, David and Noori, Mehrdad and Hakim, Gustavo A and Yazdanpanah, Moslem and Bahri, Ali and Cheraghalikhani, Milad and Dastani, Sahar and Beizaee, Farzad and Ayed, Ismail B and Desrosiers, Christian},
  journal = {Advances in neural information processing systems},
  volume  = {37},
  pages   = {48015--48044},
  year    = {2024}
}

@article{bao2025mint,
  title   = {Mint: A Simple Test-Time Adaptation of Vision-Language Models against Common Corruptions},
  author  = {Bao, Wenxuan and Deng, Ruxi and He, Jingrui},
  journal = {arXiv preprint arXiv:2510.22127},
  year    = {2025}
}

@article{xiao2025dynaprompt,
  title   = {Dynaprompt: Dynamic test-time prompt tuning},
  author  = {Xiao, Zehao and Yan, Shilin and Hong, Jack and Cai, Jiayin and Jiang, Xiaolong and Hu, Yao and Shen, Jiayi and Wang, Qi and Snoek, Cees GM},
  journal = {arXiv preprint arXiv:2501.16404},
  year    = {2025}
}

@inproceedings{chen2020graph,
  title        = {Graph optimal transport for cross-domain alignment},
  author       = {Chen, Liqun and Gan, Zhe and Cheng, Yu and Li, Linjie and Carin, Lawrence and Liu, Jingjing},
  booktitle    = {International Conference on Machine Learning},
  pages        = {1542--1553},
  year         = {2020},
  organization = {PMLR}
}

@article{chen2022plot,
  title   = {Plot: Prompt learning with optimal transport for vision-language models},
  author  = {Chen, Guangyi and Yao, Weiran and Song, Xiangchen and Li, Xinyue and Rao, Yongming and Zhang, Kun},
  journal = {arXiv preprint arXiv:2210.01253},
  year    = {2022}
}

@inproceedings{chen2026prompt,
  title     = {Prompt-ot: An optimal transport regularization paradigm for knowledge preservation in vision-language model adaptation},
  author    = {Chen, Xiwen and Zhu, Wenhui and Qiu, Peijie and Wang, Hao and Li, Huayu and Wu, Haiyu and Dong, Xuanzhao and Sotiras, Aristeidis and Wang, Yalin and Razi, Abolfazl},
  booktitle = {Proceedings of the IEEE/CVF Winter Conference on Applications of Computer Vision},
  pages     = {667--676},
  year      = {2026}
}

@inproceedings{shi2024ot,
  title     = {Ot-clip: Understanding and generalizing clip via optimal transport},
  author    = {Shi, Liangliang and Fan, Jack and Yan, Junchi},
  booktitle = {Forty-first International Conference on Machine Learning},
  year      = {2024}
}

@article{zhu2025enhancing,
  title   = {Enhancing CLIP Robustness via Cross-Modality Alignment},
  author  = {Zhu, Xingyu and Zhu, Beier and Wang, Shuo and Zhao, Kesen and Zhang, Hanwang},
  journal = {arXiv preprint arXiv:2510.24038},
  year    = {2025}
}

@article{tanguy2025sliced,
  title   = {Sliced optimal transport plans},
  author  = {Tanguy, Eloi and Chapel, Laetitia and Delon, Julie},
  journal = {arXiv preprint arXiv:2508.01243},
  year    = {2025}
}

@article{lafon2026cliptta,
  title   = {Cliptta: Robust contrastive vision-language test-time adaptation},
  author  = {Lafon, Marc and Vargas Hakim, Gustavo and Rambour, Cl{\'e}ment and Desrosiers, Christian and Thome, Nicolas},
  journal = {Advances in Neural Information Processing Systems},
  volume  = {38},
  pages   = {71849--71883},
  year    = {2026}
}

@article{mishra2024semantic,
  title   = {Semantic Anchor Transport: Robust Test-Time Adaptation for Vision-Language Models},
  author  = {Mishra, Shambhavi and Silva-Rodriguez, Julio and Ayed, Ismail Ben and Pedersoli, Marco and Dolz, Jose},
  journal = {arXiv preprint arXiv:2411.17002},
  year    = {2024}
}

@article{wu2021data,
  title   = {Data efficient language-supervised zero-shot recognition with optimal transport distillation},
  author  = {Wu, Bichen and Cheng, Ruizhe and Zhang, Peizhao and Gao, Tianren and Vajda, Peter and Gonzalez, Joseph E},
  journal = {arXiv preprint arXiv:2112.09445},
  year    = {2021}
}

@article{li2026rethinking,
  title   = {Rethinking Model Selection in VLM Through the Lens of Gromov-Wasserstein Distance},
  author  = {Li, Muyang and Liu, Yucheng and Ma, Jianbo and Osborne, Elliot and Han, Bo and Liu, Tongliang},
  journal = {arXiv preprint arXiv:2605.01325},
  year    = {2026}
}

@inproceedings{wang2026mirror,
  title        = {MIRROR: Aligning Semantic Relations from Language to Image via Gromov--Wasserstein},
  author       = {Wang, Hong-Han and Wang, Yuntao and Ding, Hu},
  booktitle    = {European Conference on Computer Vision},
  pages        = {207--225},
  year         = {2026},
  organization = {Springer}
}

@inproceedings{peyre2016gromov,
  title        = {Gromov-wasserstein averaging of kernel and distance matrices},
  author       = {Peyr{\'e}, Gabriel and Cuturi, Marco and Solomon, Justin},
  booktitle    = {International conference on machine learning},
  pages        = {2664--2672},
  year         = {2016},
  organization = {PMLR}
}

@article{vayer2020fused,
  title     = {Fused Gromov-Wasserstein distance for structured objects},
  author    = {Vayer, Titouan and Chapel, Laetitia and Flamary, R{\'e}mi and Tavenard, Romain and Courty, Nicolas},
  journal   = {Algorithms},
  volume    = {13},
  number    = {9},
  pages     = {212},
  year      = {2020},
  publisher = {MDPI}
}

@inproceedings{xu2019gromov,
  title        = {Gromov-wasserstein learning for graph matching and node embedding},
  author       = {Xu, Hongteng and Luo, Dixin and Zha, Hongyuan and Duke, Lawrence Carin},
  booktitle    = {International conference on machine learning},
  pages        = {6932--6941},
  year         = {2019},
  organization = {PMLR}
}

@article{zeng2026subspace,
  title   = {Subspace alignment for vision-language model test-time adaptation},
  author  = {Zeng, Zhichen and Bao, Wenxuan and Lin, Xiao and Qiu, Ruizhong and Wei, Tianxin and Ning, Xuying and Yan, Yuchen and Luo, Chen and Cheng, Monica Xiao and He, Jingrui and others},
  journal = {arXiv preprint arXiv:2601.08139},
  year    = {2026}
}

@article{afriat1971theory,
  title     = {Theory of maxima and the method of lagrange},
  author    = {Afriat, SN},
  journal   = {SIAM Journal on Applied Mathematics},
  volume    = {20},
  number    = {3},
  pages     = {343--357},
  year      = {1971},
  publisher = {SIAM}
}

@article{lee2026geotter,
  title   = {Geotter: Leveraging local geometry of optimal transport for zero-shot classification},
  author  = {Lee, Wei-Yang Alex and Chakraborty, Rudrasis and Lokhande, Vishnu},
  journal = {arXiv preprint arXiv:2609.13518},
  year    = {2026}
}

@article{rahimi2007random,
  title   = {Random features for large-scale kernel machines},
  author  = {Rahimi, Ali and Recht, Benjamin},
  journal = {Advances in neural information processing systems},
  volume  = {20},
  year    = {2007}
}

@book{beck2017first,
  title     = {First-order methods in optimization},
  author    = {Beck, Amir},
  year      = {2017},
  publisher = {SIAM}
}

@inproceedings{zeng2023parrot,
  title     = {Parrot: Position-aware regularized optimal transport for network alignment},
  author    = {Zeng, Zhichen and Zhang, Si and Xia, Yinglong and Tong, Hanghang},
  booktitle = {Proceedings of the ACM web conference 2023},
  pages     = {372--382},
  year      = {2023}
}

@article{spearman1961proof,
  title     = {The proof and measurement of association between two things.},
  author    = {Spearman, Charles},
  year      = {1961},
  publisher = {Appleton-Century-Crofts}
}
}

\clearpage
\appendix
\onecolumn
\begin{center}
    \Large
    \textbf{\thetitle}\\
    \vspace{0.5em}Supplementary Material\\
    \vspace{1.0em}
\end{center}
\begingroup
    \setcounter{tocdepth}{2}
    \makeatletter
    \@starttoc{apptoc}
    \makeatother
\endgroup
\vfill
\clearpage
\twocolumn

\let\appendixOriginalAddcontentsline\addcontentsline
\renewcommand{\addcontentsline}[3]{%
    \def\appendixContentsExtension{#1}%
    \def\appendixStandardExtension{toc}%
    \ifx\appendixContentsExtension\appendixStandardExtension
        \appendixOriginalAddcontentsline{apptoc}{#2}{#3}%
    \else
        \appendixOriginalAddcontentsline{#1}{#2}{#3}%
    \fi
}

\section{Low-rank Approximation of Relational InfoNCE Loss}\label{subsec:low-rank}
In this section, we introduce the low-rank approximation of the relational InfoNCE loss $\mathcal{L}_{\text{rel}}$ in Eq.~\eqref{eq:loss-adapt} to speed up \algname.
$\mathcal{L}_{\text{rel}}$ can be reformulated as follows
\begin{equation}\label{eq:rel-loss-lr-reform}
    \begin{aligned}
        \mathcal{L}_{\text{rel}}
         & =-\langle\mathbf{S}^*\otimes\mathbf{S}^*,\log\hat{\mathbf{P}}^{\theta}\rangle \\
         & =
        \underbrace{\frac1\tau
            \sum_{ii',jj'}(\mathbf{v}_i^\top\mathbf{v}_{i'}-\mathbf{t}_j^\top\mathbf{t}_{j'})^2
            \mathbf{S}^*_{ij}\mathbf{S}^*_{i'j'}}_{\text{GW term}}
        + \underbrace{\frac{1}{n^2}\sum_{ii'}
              \log(n^2\mathbf{Z}_{ii'})}_{\text{Normalization term}}
    \end{aligned}
\end{equation}
where $\mathbf{Z}_{ii'}=\sum_{jj'}\exp\left(-(\mathbf{v}_i^\top\mathbf{v}_{i'}-\mathbf{t}_j^\top\mathbf{t}_{j'})^2/\tau\right)$ is the summation of the guassian kernel
$k(x,y)=\exp(-\|x-y\|^2/\tau)$ over all text-text pairs $(j,j')$ for a given image-image pair $(i,i')$.
As shown in Appendix~\ref{subsec:proof_complexity}, the GW term can be computed relatively efficiently in $O(ncd+\min(n,c)d^2)$ time complexity, where $d$ is the embedding dimension.
The main bottleneck lies in the normalization term in Eq.~\eqref{eq:rel-loss-lr-reform}, which requires $O(n^2c^2)$ if computed directly.
To address this issue, we leverage a low-rank approximation of the guassian kernel $k(x,y)$ by an $r$-dimensional feature map $\phi_r:\mathbb{R}\to\mathbb{R}^r$ that outputs a random Fourier feature~\citep{rahimi2007random},
i.e., $k(x,y)\approx\phi_r(x)^\top\phi_r(y)$. In this way, computation of $\mathbf{Z}_{ii'}$ can be reformulated as follows
\begin{equation*}
    \begin{aligned}
        \!\!\mathbf{Z}_{ii'}
         & =\sum_{jj'}k(\mathbf{v}_i^\top\mathbf{v}_{i'},\mathbf{t}_j^\top\mathbf{t}_{j'})
        \approx\sum_{jj'}\phi_r(\mathbf{v}_i^\top\mathbf{v}_{i'})^\top\phi_r(\mathbf{t}_j^\top\mathbf{t}_{j'}) \\
         & =\phi_r(\mathbf{v}_i^\top\mathbf{v}_{i'})^\top\sum_{jj'}\phi_r(\mathbf{t}_j^\top\mathbf{t}_{j'})
        =\phi_r(\mathbf{v}_i^\top\mathbf{v}_{i'})^\top\Phi_r
    \end{aligned}
    \vspace{-5pt}
\end{equation*}
whose time complexity is reduced to $O(rc^2)$. Since $\Phi_r$ can be precomputed, the overall time complexity of computing the normalization term in \eqref{eq:rel-loss-lr-reform} is reduced to $O(r(n^2+c^2))$,
achieving a significant speedup when $r\ll c^2$ and/or $r\ll n^2$. In our experiments, we use such low-rank approximation on ImageNet-C with $r=128$.

\section{Proof}\label{sec:proof}
\subsection{Proof of Lemma~\ref{lem:bridge}}\label{subsec:proof_bridge}

\begin{lemma*} (Unification of Inference and Adaptation)\\
    The inference and adaptation of \algname\ share the same objective function $\mathcal{T}(\mathbf{S}, \theta)$ w.r.t.
    different optimized variables $\mathbf{S}\in\Pi(\boldsymbol{\mu}_v,\boldsymbol{\mu}_t)$ and $\theta$ of the image encoder as follows:
    \vspace{-10pt}
    \begin{equation}
        \begin{aligned}
            \mathcal{T}(\mathbf{S}, \theta)
            = & (1-\rho)\underbrace{\sum_{i=1}^n\sum_{j=1}^c -\mathbf{S}_{ij}\log\tilde{\mathbf{P}}^\theta_{ij}}_{\text{Semantic InfoNCE / Wasserstein}} \\
              & +\rho\underbrace{\sum_{i,i'=1}^n\sum_{j,j'=1}^c -\mathbf{S}_{ij}\mathbf{S}_{i'j'}
                         \log\hat{\mathbf{P}}^\theta_{ii',jj'}}_{\text{Relational InfoNCE / Gromov-Wasserstein}}.
        \end{aligned}
    \end{equation}
\end{lemma*}

\begin{proof}
    We prove Lemma~\ref{lem:bridge} by showing that the semantic and relational InfoNCE losses can be reformulated as Wasserstein and Gromov-Wasserstein formulations,
    respectively, plus constants independent of the transport plan $\mathbf{S}$.

    For semantic InfoNCE $\mathcal L_{\mathrm{sem}}$, since $\mathbf{S}$ is a probability distribution, i.e.,
    $\sum_{i,j}\mathbf{S}_{ij}=1$, with uniform marginals $\sum_j\mathbf{S}_{ij}=\frac{1}{n}$ and $\sum_i\mathbf{S}_{ij}=\frac{1}{c}$, expanding $\tilde{\mathbf P}^{\theta}$ yields
    \begin{equation}
        \begin{aligned}
            \mathcal{L}_{\mathrm{sem}}
             & = -\sum_{i,j} \mathbf{S}_{ij}\log\tilde{\mathbf{P}}_{ij}^{\theta} \\
             & = -\sum_{i,j}\mathbf{S}_{ij}
            \log\frac{
                    \exp(\mathbf{v}_i^\top\mathbf{t}_j/\tau)
                }{
                    n\sum_{k}\exp(\mathbf{v}_i^\top\mathbf{t}_k/\tau)
                }             \\
             & = -\frac{1}{\tau}
            \sum_{i,j}\mathbf{v}_i^\top\mathbf{t}_j \mathbf{S}_{ij}
            +\sum_{i,j}\mathbf{S}_{ij}
            \log\left(
            n\sum_k
            \exp(\frac{\mathbf{v}_i^\top\mathbf{t}_k}{\tau})
            \right)                                                              \\
             & = \frac{1}{\tau}
            \underset{\text{Wasserstein distance}}{\underbrace{
                    \sum_{i,j}(1-\mathbf{v}_i^\top\mathbf{t}_j)\mathbf{S}_{ij}}} \\
             & \quad+
            \underset{c_\mathrm{w}}{\underbrace{
                    \sum_i\frac{1}{n}
                    \log\left(
                    n\sum_k
                    \exp(\frac{\mathbf{v}_i^\top\mathbf{t}_k}{\tau})
                    \right)
                    -\frac{1}{\tau}
                }}
        \end{aligned}
    \end{equation}
    In this way, we have reformulated the semantic InfoNCE loss in Eq.~\eqref{eq:loss-sem} as a Wasserstein distance in Eq.~\eqref{eq:fgw-dist}
    with $\mathbf{M}_{ij}=\sqrt{1-\mathbf{v}_i^\top\mathbf{t}_j}$ plus a constant $c_\mathrm{w}$ independent of $\mathbf S$.

    Similarly, for relational InfoNCE, expanding $\hat{\mathbf P}^{\theta}$ yields

    \begin{equation}
        \begin{aligned}
             & \mathcal{L}_{\mathrm{rel}}=
            -\!\!\sum_{ii',jj'}
            \mathbf{S}_{ij}\mathbf{S}_{i'j'}
            \log\widehat{\mathbf{P}}_{ii',jj'}^{\theta}                        \\
             & = -\!\!\sum_{ii',jj'}
            \mathbf{S}_{ij}\mathbf{S}_{i'j'}
            \log\frac{
                    \exp\left(
                    -\frac{(\mathbf{v}_i^\top\mathbf{v}_{i'}
                        -\mathbf{t}_j^\top\mathbf{t}_{j'})^2}{\tau}
                    \right)
                }{
                    n^2\sum_{k,k'}
                    \exp\left(
                    -\frac{(\mathbf{v}_i^\top\mathbf{v}_{i'}
                        -\mathbf{t}_k^\top\mathbf{t}_{k'})^2}{\tau}
                    \right)
                }             \\
             & = \frac{1}{\tau}\sum_{ii',jj'}
            (\mathbf{v}_i^\top\mathbf{v}_{i'}
            -\mathbf{t}_j^\top\mathbf{t}_{j'})^2
            \mathbf{S}_{ij}\mathbf{S}_{i'j'}                                   \\
             & \quad+\!\!
            \sum_{ii',jj'}
            \mathbf{S}_{ij}\mathbf{S}_{i'j'}
            \log\left(
            n^2\sum_{k,k'}
            \exp\left(
            -\frac{(\mathbf{v}_i^\top\mathbf{v}_{i'}
                 -\mathbf{t}_k^\top\mathbf{t}_{k'})^2}{\tau}
            \right)
            \right)                                                            \\
             & = \frac{1}{\tau}\underset{\text{Gromov-Wasserstein}}{
                                   \underbrace{\sum_{ii',jj'}
                                       (\mathbf{v}_i^\top\mathbf{v}_{i'}
                                       -\mathbf{t}_j^\top\mathbf{t}_{j'})^2
                                       \mathbf{S}_{ij}\mathbf{S}_{i'j'}    } } \\
             & \quad+
            \underset{c_\mathrm{gw}}{\underbrace{
                    \frac{1}{n^2}\!\sum_{i,i'}
                    \log\left(
                    n^2\sum_{k,k'}
                    \exp\left(
                    -\frac{(\mathbf{v}_i^\top\mathbf{v}_{i'}
                        -\mathbf{t}_k^\top\mathbf{t}_{k'})^2}{\tau}
                    \right)
                    \right)
                }}
        \end{aligned}
    \end{equation}

    In this way, we have reformulated the relational InfoNCE loss in Eq.~\eqref{eq:loss-rel} as a Gromov-Wasserstein distance in Eq.~\eqref{eq:fgw-dist}
    with $\mathbf{C}^v_{ii'}=1-\mathbf{v}_i^\top\mathbf{v}_{i'}$ and $\mathbf{C}^t_{jj'}=1-\mathbf{t}_j^\top\mathbf{t}_{j'}$
    plus a constant $c_\mathrm{gw}$ independent of $\mathbf S$.

    Therefore, when $\rho=\alpha=\beta$,
    the unified objective in Eq.~(8) becomes
    \begin{equation}
        \begin{aligned}
            \mathcal T(\mathbf S,\theta)
             & =\frac{1}{\tau}\mathrm{FGW}_{2,\rho}
            (\mathbf S,\theta)+c_{\mathrm{fgw}},
        \end{aligned}
    \end{equation}
    where
    $c_{\mathrm{fgw}}=(1-\rho)c_{\mathrm{w}}+\rho c_{\mathrm{gw}}$.
    Since $c_{\mathrm{fgw}}$ is independent of $\mathbf S$,
    minimizing $\mathcal T$ over $\mathbf S$ is
    equivalent to solving the FGW formulation in Eq.~\eqref{eq:fgw-dist}, i.e.,
    \begin{equation}
        \mathbf S^*=\arg\min_{\mathbf S\in\Pi(\boldsymbol{\mu}_v,\boldsymbol{\mu}_t)} \mathrm{FGW}_{2,\rho} = \arg\min_{\mathbf S\in\Pi(\boldsymbol{\mu}_v,\boldsymbol{\mu}_t)}\mathcal T(\mathbf S,\theta).
    \end{equation}
    Meanwhile, fixing the optimal transport plan
    $\mathbf S^*$ gives the multi-level InfoNCE
    loss in Eq.~\eqref{eq:loss-adapt}.
    Therefore, both stages optimize the same unified
    objective with respect to different variables $\mathbf S$ and $\theta$.
    In this way, we have proven Lemma~\ref{lem:bridge}.
\end{proof}

\subsection{Proof of Theorem~\ref{thm:convergence}}\label{subsec:proof_convergence}
\begin{theorem*}(Convergence of \algname)\\
    Let $\mathbf{S}^{(k)}$ and $\theta^{(k)}$ denote the transport plan and the parameters of the image encoder at the $k$-th iteration,
    the unified objective $\mathcal{T}(\mathbf{S}, \theta)$ in Eq.~\eqref{eq:bridge} is non-increasing w.r.t. $k$, i.e.,
    \begin{equation}
        \mathcal{T}(\mathbf{S}^{(k+1)}, \theta^{(k+1)})\leq\mathcal{T}(\mathbf{S}^{(k)}, \theta^{(k)})
    \end{equation}
\end{theorem*}

\begin{proof}
    We prove the convergence of \algname\ on a fixed batch during the alternating optimization of
    $\mathcal{T}(\mathbf{S},\theta)$ w.r.t. $\mathbf{S}$ and $\theta$.
    Assume that the FGW inference step is descent-preserving
    and that $\mathcal{T}(\mathbf{S},\cdot)$ has an
    $L$-Lipschitz continuous gradient, with adaptation
    learning rate $0<\eta\leq 1/L$.

    For the inference stage at the $k$-th iteration, $\mathcal{T}(\mathbf{S},\theta)$
    is concave w.r.t. $\mathbf{S}$ over the transport polytope~\citep{peyre2016gromov,xu2019gromov}.
    Consequently, its linear approximation at the current transport plan upper-bounds the original objective.
    Minimizing this approximation w.r.t. $\mathbf{S}$ with KL-proximal regularization therefore guarantees that the FGW objective is non-increasing, i.e.,
    \begin{equation}
        \begin{aligned}
             & \mathcal{T}(\mathbf{S}^{(k+1)},\theta^{(k)})    \\
             & \leq \mathcal{T}(\mathbf{S}^{(k)},\theta^{(k)})
            +\left\langle
            \nabla_{\mathbf{S}}\mathcal{T}
            (\mathbf{S}^{(k)},\theta^{(k)}),
            \mathbf{S}^{(k+1)}-\mathbf{S}^{(k)}
            \right\rangle                                      \\
             & \leq \mathcal{T}(\mathbf{S}^{(k)},\theta^{(k)})
            -\lambda D_{\mathrm{KL}}
            (\mathbf{S}^{(k+1)}\|\mathbf{S}^{(k)})             \\
             & \leq \mathcal{T}(\mathbf{S}^{(k)},\theta^{(k)})
        \end{aligned}
    \end{equation}

    For the adaptation stage at the $k$-th iteration, assume that
    $\mathcal{T}(\mathbf{S}^{(k+1)},\cdot)$ has an continuous gradient that is
    $L$-Lipschitz. With gradient
    descent and learning rate $0<\eta\leq 1/L$,
    the descent lemma in \citep{beck2017first} gives
    \begin{equation}
        \mathcal{T}(\mathbf{S}^{(k+1)},\theta^{(k+1)})
        \leq
        \mathcal{T}(\mathbf{S}^{(k+1)},\theta^{(k)})
    \end{equation}

    Combining the two inequalities yields
    \begin{equation}
        \mathcal{T}(\mathbf{S}^{(k+1)},\theta^{(k+1)})
        \leq
        \mathcal{T}(\mathbf{S}^{(k)},\theta^{(k)})
    \end{equation}
    Finally, $\mathcal{T}(\mathbf{S},\theta)$ is lower bounded by zero since itis a nonnegative
    weighted sum of cross-entropy losses. Therefore, for a fixed batch,
    the objective function $\mathcal{T}(\mathbf{S},\theta)$ of \algname\ is non-increasing and converges to a finite limit.
    In this way, we have proven Theorem~\ref{thm:convergence}.
\end{proof}

\subsection{Proof of Theorem~\ref{thm:complexity}}\label{subsec:proof_complexity}
\begin{theorem*}(Complexity of \algname)\\
    The per-batch time complexity of the proposed \algname\ is $O\left(K\left(ncd+\min(n,c)d^2+Jnc\right)+r(n^2+c^2)\right)$, and the per-batch space complexity is $O(nc+nd+cd)$,
    where $n$ is the image batch size, $c$ is the number of text classes, $d$ is the embedding dimension, $K$ is the number of iterations for the KL-proximal gradient method, $J$ is the number of iterations for the Sinkhorn algorithm,
    and $r$ is the rank for low-rank approximation of the relational InfoNCE loss.
\end{theorem*}

\begin{proof}
    We analyze the per-batch time and space complexities
    of ORIGIN, excluding VLM forward/backward propagation the storage of VLM parameters.

    \paragraph{Time complexity.}
    For the inference stage,
    given the cross-modal cost matrix $\mathbf{M}$ and the intra-modal relational cost matrices $\mathbf{C}^v$ and $\mathbf{C}^t$,
    the Wasserstein term is computed in $O(nc)$ for each KL-proximal iteration.
    For the GW term, using efficient matrix multiplication~\cite{zeng2023parrot, yu2025joint}, the GW gradient requires $O(ncd+\min(n,c)d^2)$ time.
    Using $J$ Sinkhorn iterations with time complexity $O(Jnc)$ per KL-proximal iteration,
    the overall time complexity of the FGW inference stage is
    \begin{equation}
        O(K[ncd+\min(n,c)d^2+Jnc])
    \end{equation}

    For the adaptation stage, the semantic InfoNCE requires $O(nc)$ time using pre-computed cross-modal cost $\mathbf{M}$.
    As shown in Appendix~\ref{subsec:low-rank}, the relational InfoNCE loss consists of a GW term and a normalization term,
    which requires $O(ncd+\min(n,c)d^2)$ and $O(r(n^2+c^2))$ time, respectively.
    Since $K\geq 1$, combining both stages gives the total per-batch time complexity of \algname\ as
    \begin{equation}
        O\left(
        K[ncd+\min(n,c)d^2+Jnc]
        +r(n^2+c^2)
        \right).
    \end{equation}

    \paragraph{Space complexity.}
    Storing the image and text embeddings requires $O(nd+cd)$ space.
    The transport plan $\mathbf{S}$ and the cross-modal cost matrix $\mathbf{M}$ require $O(nc)$ space.
    While the GW term generally involves two intra-modal feature matrices typically of $O(n^2)$ and $O(c^2)$ space, respectively,
    we can avoid materializing these matrices by exploiting feature factorization and appropriate matrix multiplication ordering~\citep{peyre2016gromov}.
    Furthermore, chunked computation of relational InfoNCE avoids materializing the full pairwise similarity matrix of size $\mathbb{R}^{n^2 \times c^2}$ and the low-rank feature matrices.
    In this way, the overall space complexity of \algname\ is
    \begin{equation}
        O(nc+nd+cd),
    \end{equation}
    In this way, we have proven Theorem~\ref{thm:complexity}.
\end{proof}

\section{Detailed Experimental Setup}\label{sec:exp_setup}

\subsection{Datasets \& Metrics.}
We adopt three standard benchmarks for corrupted image classification: CIFAR10-C, CIFAR100-C, and ImageNet-C~\citep{hendrycks2019benchmarking},
which share the same 15 different types of corruptions from CIFAR10, CIFAR100~\citep{krizhevsky2009learning}, and ImageNet~\citep{deng2009imagenet}, respectively.
Following standard TTA protocols~\citep{wang2020tent,zeng2026subspace}, all results reported are based on the highest severity level (Level 5).
We adopt the classification accuracy as the evaluation metric~\citep{karmanov2024efficient, zhang2024dual, dobler2024lost, farina2024frustratingly, maharana2025batclip}.

\subsection{VLM Backbones.} We adopt the CLIP~\citep{radford2021learning} model with different visual encoders, including ViT-B-16 and ViT-B-32.
ViT-B-16 is adopted as the default visual encoder unless noted otherwise.

\subsection{VLM TTA Baselines.}
We compare \algname\ against the following state-of-the-art VLM TTA approaches.
For training-free methods which relies on memory management or data augmentation,
we include ECALP~\citep{li2024efficient} and TENT~\citep{wang2020tent}.
For training-based methods that optimizes prompts or tunable parameters inside VLMs during adaptation,
we include RoTTA~\citep{yuan2023robust}, TPT~\citep{shu2022test}, MEMO~\citep{zhang2022memo}, WATT~\citep{osowiechi2024watt},
MINT~\citep{bao2025mint}, BATCLIP~\citep{maharana2025batclip}, SubTTA~\citep{zeng2026subspace}, and SAT~\citep{mishra2024semantic}.

\begin{table}[t]
    \vspace{-10pt}
    \centering
    \setlength\tabcolsep{5pt}
    \caption{Detailed hyperparameter settings of \algname.}
    \vspace{-5pt}
    \begin{tabular}{lcc}
        \toprule

        Hyperparameters & $\alpha$ & $\tau$/$\epsilon$ \\
        \midrule
        CIFAR10-C       & 0.25     & 0.001             \\
        CIFAR100-C      & 0.5      & 0.01              \\
        ImageNet-C      & 0.25     & 0.1               \\

        \bottomrule
    \end{tabular}
    \label{tab:param}
    \vspace{-5pt}
\end{table}

\begin{figure}[t]
    \centering
    \begin{tcolorbox}[
            title={Class Templates},
            colback=black!2,
            colframe=black!75,
            colbacktitle=black!75,
            coltitle=white,
            fonttitle=\large,
            fontupper=\normalsize,
            boxrule=0.8pt,
            arc=2mm,
            left=7pt,right=7pt,top=5pt,bottom=5pt
        ]
        \textbf{T1:} A photo of a \texttt{\{class\}}\\[2pt]
        \textbf{T2:} itap of a \texttt{\{class\}}\\[2pt]
        \textbf{T3:} a bad photo of the \texttt{\{class\}}\\[2pt]
        \textbf{T4:} a origami \texttt{\{class\}}\\[2pt]
        \textbf{T5:} a photo of the large \texttt{\{class\}}\\[2pt]
        \textbf{T6:} a \texttt{\{class\}} in a video game\\[2pt]
        \textbf{T7:} art of the \texttt{\{class\}}\\[2pt]
        \textbf{T8:} a photo of the small \texttt{\{class\}}
    \end{tcolorbox}
    \vspace{-10pt}
    \caption{Class templates T1-T8 used in the sensitivity study. The placeholder \texttt{\{class\}} is replaced by the corresponding class name.}
    \label{fig:exp_templates}
    \vspace{-10pt}
\end{figure}

\subsection{Reproducibility.}
For all experiments, the reported results are averaged against 5 different runs with randomized data loader.
To ensure a fair comparison, all baselines are evaluated on a dataset under the same batch size $b$, which is 200 for CIFAR10-C and CIFAR100-C,
and 50 for ImageNet-C. Hyperparameters of all baselines VLM TTA methods are set as default in their official implementations.
Detailed hyperparameter values of \algname\ are listed in the following Table~\ref{tab:param}.
For all training-based TTA methods, the AdamW optimizer is used with a learning rate of 5e-3 for CIFAR10-C and CIFAR100-C, and 1e-3 for ImageNet-C.
We use T1 in Figure~\ref{fig:exp_templates} as the default class template.

\subsection{Machine.} All experiments are conducted on a server with dual Intel® Xeon® Gold 6240R CPUs and 4 NVIDIA Tesla V100-SXM2 32GB GPUs.

\subsection{Setup in Figure~\ref{fig:spearman}.}
For the robustness study in Figure~\ref{fig:spearman-panel},
we compute the Spearman's rank correlation coefficient~\citep{spearman1961proof} before and after distribution shifts,
measuring the preservation of rank consistency which directly affect the performance of zero-shot image classification of CLIP modeled as a ranking task.
Formally, the Spearman's rank correlation coefficient is defined as
\begin{equation}\label{eq:spearman}
    \varrho(X,Y)=
    \frac{
        \sum_{i=1}^{n}(R(x_i)-\bar R_X)(R(y_i)-\bar R_Y)
    }{
        \sqrt{\sum_{i=1}^{n}(R(x_i)-\bar R_X)^2}
        \sqrt{\sum_{i=1}^{n}(R(y_i)-\bar R_Y)^2}
    }
\end{equation}
where $R(x_i)$ and $R(y_i)$ denote the ranks of $x_i$ and $y_i$ in the two sets of variables $X=\{x_1,\ldots,x_n\}$ and $Y=\{y_1,\ldots,y_n\}$,
and $\bar R_X$ and $\bar R_Y$ denote the mean ranks of $X$ and $Y$, respectively.
Specifically, for intra-modal consistency, the Spearman's rank correlation coefficient is computed as $\varrho(\mathbf{V}^{(0)\top}\mathbf{V}^{(0)}, \mathbf{V}^{(s)\top}\mathbf{V}^{(s)})$,
where $\mathbf{V}^{(0)}$ and $\mathbf{V}^{(s)}$ are the image embeddings of the original and corrupted images at shift level $s$, respectively.
Similarly, for inter-modal consistency, the Spearman's rank correlation coefficient is computed as $\varrho(\mathbf{V}^{(0)\top}\mathbf{T}^{(0)}, \mathbf{V}^{(s)\top}\mathbf{T}^{(0)})$,
where $\mathbf{T}^{(0)}$ is the text embeddings.
Larger $\varrho$ indicates better preservation of rank consistency after distribution shifts, indicating the robustness of different signals.
Experiments are conducted on CIFAR10-C with ViT-B-16 as the visual encoder, averaged again all 15 corruption types.
Note that we avoid directly using the embedding similarity to measure the robustness of the modelas CLIP is explicitly pre-trained to minimize pairwise embedding distances,
thus biasing in favor of cross-modal similarity.

For the relational discrepancy study in Figure~\ref{fig:rel-hist}, we measure the relational discrepancy across two modalities for correct and incorrect predictions, respectively.
Correct predictions are defined as predicting both image-text pairs right in a pair of cross-modal relation, while incorrect predictions are defined as predicting at least one of the two pairs wrong.
Macro AUC measures the ability of relational discrepancy
in distinguishing correct from incorrect predictions, defined as
\begin{equation}
    \begin{aligned}
        \mathrm{MacroAUC} & = \frac{1}{c}\sum_{j=1}^{c}\mathrm{AUC}_j, \\
        \mathrm{AUC}_j    & = \int \mathrm{TPR}_j(t)\,
        d\mathrm{FPR}_j(t),
    \end{aligned}
\end{equation}
where $c$ denotes the number of classes, and $\mathrm{TPR}_j$ and
$\mathrm{FPR}_j$ denote the true-positive and false-positive rates
obtained by thresholding the negative relational discrepancy for class $j$.
A higher Macro AUC indicates stronger discrimination between correct
and incorrect predictions.
Experiments are conducted on CIFAR100-C with ViT-B-16 as the visual encoder.

\section{Additional Experimental Results}\label{sec:add_exp_results}

\begin{table*}[t]
    \centering
    \begingroup
    \definecolor{benchfirst}{HTML}{F5C6C6}
    \definecolor{benchsecond}{HTML}{FAD6AC}
    \definecolor{benchthird}{HTML}{CBE0F3}
    \newcommand{\benchFirst}[1]{\cellcolor{benchfirst}#1}
    \newcommand{\benchSecond}[1]{\cellcolor{benchsecond}#1}
    \newcommand{\benchThird}[1]{\cellcolor{benchthird}#1}
    \setlength{\tabcolsep}{2.5pt}
    \renewcommand{\arraystretch}{1.05}
    \caption{Benchmarking results with ViT-B-32. The 1st/2nd/3rd best results are shaded \colorbox{benchfirst}{\strut red}/\colorbox{benchsecond}{\strut orange}/\colorbox{benchthird}{\strut blue}, respectively.}
    \vspace{-5pt}
    \label{tab:exp_bench_vit_b32}
    \resizebox{\textwidth}{!}{%
        \begin{tabular}{@{}cl*{16}{c}@{}}
            \toprule
            \multicolumn{2}{l}{\multirow{2}{*}{\textbf{Method}}}
                                                                  & \multicolumn{3}{c}{\textbf{Noise}}   & \multicolumn{4}{c}{\textbf{Blur}}
                                                                  & \multicolumn{4}{c}{\textbf{Weather}} & \multicolumn{4}{c}{\textbf{Digital}}
                                                                  & \multirow{2}{*}{\textbf{Mean}}                                                                                                                                                                                                                                                                                                                                                                                        \\
            \cmidrule(lr){3-5}\cmidrule(lr){6-9}\cmidrule(lr){10-13}\cmidrule(lr){14-17}
                                                                  &                                      & \textbf{Gauss.}                      & \textbf{Shot}       & \textbf{Impul.}     & \textbf{Defoc.}     & \textbf{Glass}      & \textbf{Motion}     & \textbf{Zoom}
                                                                  & \textbf{Snow}                        & \textbf{Frost}                       & \textbf{Fog}        & \textbf{Brit.}      & \textbf{Contr.}     & \textbf{Elast.}     & \textbf{Pixel.}     & \textbf{JPEG}       &                                                                                                                                                                                                     \\
            \midrule
            \multirow{11}{*}{\rotatebox{90}{\textbf{CIFAR10-C}}}  & Source                               & 35.98                                & 40.01               & 43.53               & 70.10               & 41.37               & 64.32               & 69.75               & 70.77               & 72.60               & 66.65               & 81.33               & 64.59               & 59.35               & 48.65               & 56.94               & 59.33               \\
                                                                  & ECALP                                & 44.39                                & 47.26               & 44.16               & 70.75               & 43.19               & 67.90               & 74.16               & 72.75               & 74.72               & 67.67               & 82.21               & 62.93               & 61.50               & 47.72               & 55.80               & 60.92               \\
                                                                  & RoTTA                                & 36.70                                & 41.13               & 43.70               & 70.21               & 42.85               & 64.69               & 69.99               & 71.18               & 72.64               & 66.85               & 81.77               & 64.74               & 60.27               & 49.06               & 57.26               & 59.14               \\
                                                                  & TPT                                  & 43.62                                & 46.79               & 48.03               & 71.13               & 47.81               & 66.50               & 72.04               & 73.51               & 76.40               & 68.55               & 84.01               & 66.37               & 64.12               & 51.35               & 58.11               & 62.42               \\
                                                                  & MEMO                                 & 36.73                                & 40.59               & 44.10               & 71.22               & 42.31               & 64.73               & 70.21               & 72.08               & 73.38               & 68.11               & 82.17               & 64.45               & 60.95               & 47.53               & 57.36               & 59.64               \\
                                                                  & WATT                                 & 43.59                                & 49.05               & 49.45               & 72.09               & 46.36               & 67.03               & 71.37               & 73.94               & 74.15               & 71.23               & 83.87               & 72.76               & 61.63               & 58.76               & \benchThird{63.23}  & 63.63               \\
                                                                  & MINT                                 & 54.36                                & 57.78               & 47.50               & 73.31               & 55.90               & 73.98               & 76.06               & 74.45               & 74.64               & 70.53               & 85.58               & 70.90               & 64.47               & 59.15               & 59.45               & 66.42               \\
                                                                  & BATCLIP                              & 51.85                                & 56.03               & 54.23               & 76.05               & 55.24               & 74.62               & 75.49               & 76.75               & 78.45               & 75.22               & 86.26               & 77.03               & 67.05               & 57.84               & 61.56               & 68.25               \\
                                                                  & SubTTA                               & \benchThird{60.24}                   & \benchThird{63.41}  & \benchThird{57.03}  & \benchThird{78.30}  & \benchThird{61.80}  & \benchThird{76.49}  & \benchThird{77.15}  & \benchThird{80.10}  & \benchThird{79.86}  & \benchThird{78.57}  & \benchFirst{89.01}  & \benchSecond{83.34} & \benchThird{70.97}  & \benchThird{74.24}  & 63.01               & \benchThird{72.50}  \\
                                                                  & SAT                                  & \benchSecond{62.87}                  & \benchSecond{65.42} & \benchSecond{60.23} & \benchSecond{79.66} & \benchSecond{66.35} & \benchSecond{79.07} & \benchSecond{80.42} & \benchSecond{81.63} & \benchSecond{82.04} & \benchSecond{80.28} & \benchThird{88.66}  & \benchThird{82.49}  & \benchSecond{74.17} & \benchSecond{74.45} & \benchSecond{67.24} & \benchSecond{75.77} \\
                                                                  & \algname                             & \benchFirst{67.31}                   & \benchFirst{70.26}  & \benchFirst{63.08}  & \benchFirst{81.65}  & \benchFirst{68.45}  & \benchFirst{81.53}  & \benchFirst{82.87}  & \benchFirst{83.67}  & \benchFirst{83.45}  & \benchFirst{84.10}  & \benchSecond{88.95} & \benchFirst{86.41}  & \benchFirst{75.04}  & \benchFirst{79.28}  & \benchFirst{72.88}  & \benchFirst{77.93}  \\
            \midrule
            \multirow{11}{*}{\rotatebox{90}{\textbf{CIFAR100-C}}} & Source                               & 15.77                                & 18.11               & 17.32               & 39.26               & 17.45               & 38.22               & 43.52               & 42.21               & 43.26               & 39.61               & 50.81               & 29.55               & 29.20               & 23.27               & 29.58               & 31.53               \\
                                                                  & ECALP                                & 18.68                                & 19.24               & 17.35               & 38.25               & 17.77               & 39.90               & 44.32               & 43.68               & 44.92               & 41.01               & 52.67               & 31.21               & 30.67               & 23.05               & 30.08               & 33.24               \\
                                                                  & RoTTA                                & 16.40                                & 18.84               & 18.18               & 38.69               & 17.57               & 39.00               & 43.25               & 42.82               & 43.80               & 39.67               & 50.96               & 28.97               & 29.08               & 23.68               & 30.25               & 32.23               \\
                                                                  & TPT                                  & 16.52                                & 17.53               & 16.96               & 39.09               & 19.29               & 38.85               & 43.91               & 43.40               & 44.75               & 40.31               & 50.75               & 27.59               & 30.54               & 23.77               & 29.25               & 32.36               \\
                                                                  & MEMO                                 & 16.87                                & 18.08               & 18.14               & 41.27               & 17.15               & 40.56               & 45.49               & 44.27               & 44.43               & 40.46               & 52.57               & 29.72               & 30.91               & 22.60               & 31.25               & 33.09               \\
                                                                  & WATT                                 & 21.66                                & 22.00               & 23.72               & \benchThird{47.79}  & 18.33               & 44.32               & 49.89               & \benchThird{47.72}  & 47.52               & 45.19               & 57.10               & \benchThird{44.13}  & 32.82               & 29.85               & 35.35               & 38.17               \\
                                                                  & MINT                                 & 23.70                                & 26.27               & 21.48               & 45.97               & 23.52               & 42.27               & 48.51               & 46.52               & 45.19               & 43.27               & 55.68               & 38.65               & 33.82               & 28.25               & 32.14               & 37.16               \\
                                                                  & BATCLIP                              & 21.45                                & 24.99               & 22.62               & 46.59               & 22.94               & \benchThird{44.42}  & 49.46               & 47.01               & 46.55               & 44.75               & \benchThird{58.13}  & 38.56               & 34.62               & 28.59               & 33.44               & 37.29               \\
                                                                  & SubTTA                               & \benchThird{25.86}                   & \benchThird{26.85}  & \benchThird{24.17}  & 46.17               & \benchThird{26.18}  & 43.73               & \benchThird{50.55}  & 47.42               & \benchThird{47.87}  & \benchThird{46.20}  & 57.05               & 41.77               & \benchThird{37.14}  & \benchThird{32.71}  & \benchThird{35.80}  & \benchThird{38.97}  \\
                                                                  & SAT                                  & \benchSecond{30.85}                  & \benchSecond{33.37} & \benchSecond{29.55} & \benchFirst{51.72}  & \benchSecond{33.03} & \benchFirst{51.36}  & \benchSecond{54.45} & \benchSecond{52.45} & \benchSecond{52.04} & \benchSecond{52.08} & \benchFirst{61.62}  & \benchSecond{52.19} & \benchSecond{42.20} & \benchSecond{42.15} & \benchSecond{39.62} & \benchSecond{46.01} \\
                                                                  & \algname                             & \benchFirst{36.16}                   & \benchFirst{37.81}  & \benchFirst{32.54}  & \benchSecond{50.41} & \benchFirst{36.92}  & \benchSecond{50.47} & \benchFirst{54.63}  & \benchFirst{52.64}  & \benchFirst{52.82}  & \benchFirst{52.70}  & \benchSecond{60.80} & \benchFirst{55.43}  & \benchFirst{44.60}  & \benchFirst{47.19}  & \benchFirst{41.50}  & \benchFirst{47.11}  \\
            \midrule
            \multirow{11}{*}{\rotatebox{90}{\textbf{ImageNet-C}}} & Source                               & 13.22                                & 13.12               & 12.37               & 24.52               & 11.43               & 23.12               & 20.61               & 25.35               & 26.30               & 30.48               & \benchThird{50.89}  & 17.03               & 19.05               & 32.14               & 29.49               & 23.19               \\
                                                                  & ECALP                                & 14.25                                & 15.47               & 15.03               & 23.10               & 12.64               & 23.25               & 20.98               & 24.12               & 25.40               & 29.68               & 48.41               & 17.97               & 19.42               & 31.31               & 28.87               & 23.13               \\
                                                                  & RoTTA                                & 13.47                                & 13.69               & 13.54               & 24.23               & 12.48               & 22.98               & 20.03               & 25.49               & 26.36               & 30.23               & 49.98               & 17.57               & 19.39               & 32.08               & 29.45               & 23.47               \\
                                                                  & TPT                                  & 12.52                                & 12.29               & 12.57               & 24.99               & 12.09               & 22.93               & 20.68               & 26.78               & 26.58               & 30.12               & \benchThird{50.89}  & 16.87               & 19.90               & 33.09               & 30.12               & 24.02               \\
                                                                  & MEMO                                 & 12.40                                & 13.44               & 13.10               & 24.91               & 11.46               & 22.71               & 20.06               & 25.45               & 25.80               & 30.35               & 50.11               & 17.50               & 19.00               & 32.15               & 28.96               & 23.53               \\
                                                                  & WATT                                 & 14.47                                & 14.22               & 14.18               & \benchSecond{26.11} & 15.53               & 25.32               & 22.30               & 26.38               & 25.43               & 31.16               & 50.35               & \benchThird{22.58}  & 19.45               & 33.13               & 30.34               & 24.69               \\
                                                                  & MINT                                 & \benchThird{18.73}                   & 19.52               & 19.82               & 24.67               & 18.76               & 27.53               & 22.84               & 27.65               & 27.55               & 32.35               & \benchSecond{50.95} & 20.61               & 23.56               & 34.72               & 33.45               & 26.85               \\
                                                                  & BATCLIP                              & 17.55                                & 18.76               & 16.75               & 24.90               & 20.46               & \benchThird{28.03}  & \benchThird{24.33}  & \benchThird{29.13}  & 27.40               & \benchThird{35.70}  & 49.38               & 20.20               & \benchThird{27.98}  & \benchThird{35.72}  & \benchSecond{34.47} & \benchThird{27.45}  \\
                                                                  & SubTTA                               & \benchFirst{23.23}                   & \benchThird{23.04}  & \benchSecond{23.53} & \benchThird{25.87}  & \benchSecond{23.58} & \benchSecond{28.14} & \benchSecond{26.99} & \benchSecond{30.72} & \benchFirst{31.70}  & \benchFirst{39.36}  & \benchFirst{50.96}  & \benchSecond{25.64} & \benchSecond{33.18} & \benchSecond{36.76} & \benchThird{34.42}  & \benchSecond{30.29} \\
                                                                  & SAT                                  & 18.25                                & \benchFirst{23.66}  & \benchFirst{24.50}  & 24.58               & \benchThird{21.27}  & 25.64               & 23.59               & 26.93               & \benchThird{28.03}  & 28.44               & 41.44               & 16.98               & 25.07               & 25.68               & 25.11               & 25.49               \\
                                                                  & \algname                             & \benchSecond{22.14}                  & \benchSecond{23.24} & \benchThird{23.36}  & \benchFirst{26.42}  & \benchFirst{25.30}  & \benchFirst{30.22}  & \benchFirst{28.94}  & \benchFirst{32.16}  & \benchSecond{30.06} & \benchSecond{37.36} & 47.72               & \benchFirst{27.72}  & \benchFirst{34.14}  & \benchFirst{38.68}  & \benchFirst{36.54}  & \benchFirst{30.93}  \\
            \bottomrule
        \end{tabular}%
    }
    \endgroup
    \vspace{-5pt}
\end{table*}

\subsection{Additional Benchmarking Results}
We include additional benchmarking results of \algname\ with ViT-B-32 in Table~\ref{tab:exp_bench_vit_b32}. The observations aligns with that of ViT-B-16 in Table~\ref{tab:exp_bench_vit_b16}:
\textbf{\algname\ consistently outperforms all state-of-the-art VLM TTA methods across all three benchmarks.}
Notably, \algname\ achieves a significant improvement of 5.4\% on most difficult benchmarks, ImageNet-C,
compared to SAT~\citep{mishra2024semantic} which is the best-performing baseline with performances close to \algname on the other two benchmarks, CIFAR10-C and CIFAR100-C.
While SubTTA~\citep{zeng2026subspace} achieve a similar performance to \algname\ on ImageNet-C using ViT-B-32,
it is much less effective on CIFAR10-C and CIFAR100-C and with a performance gap of up to 8.2\% compared to \algname.

\subsection{On the Recovery of Semantics and Relations of VLMs}

\noindent\begin{minipage}{\columnwidth}
    \centering
    \includegraphics[width=0.8\linewidth]{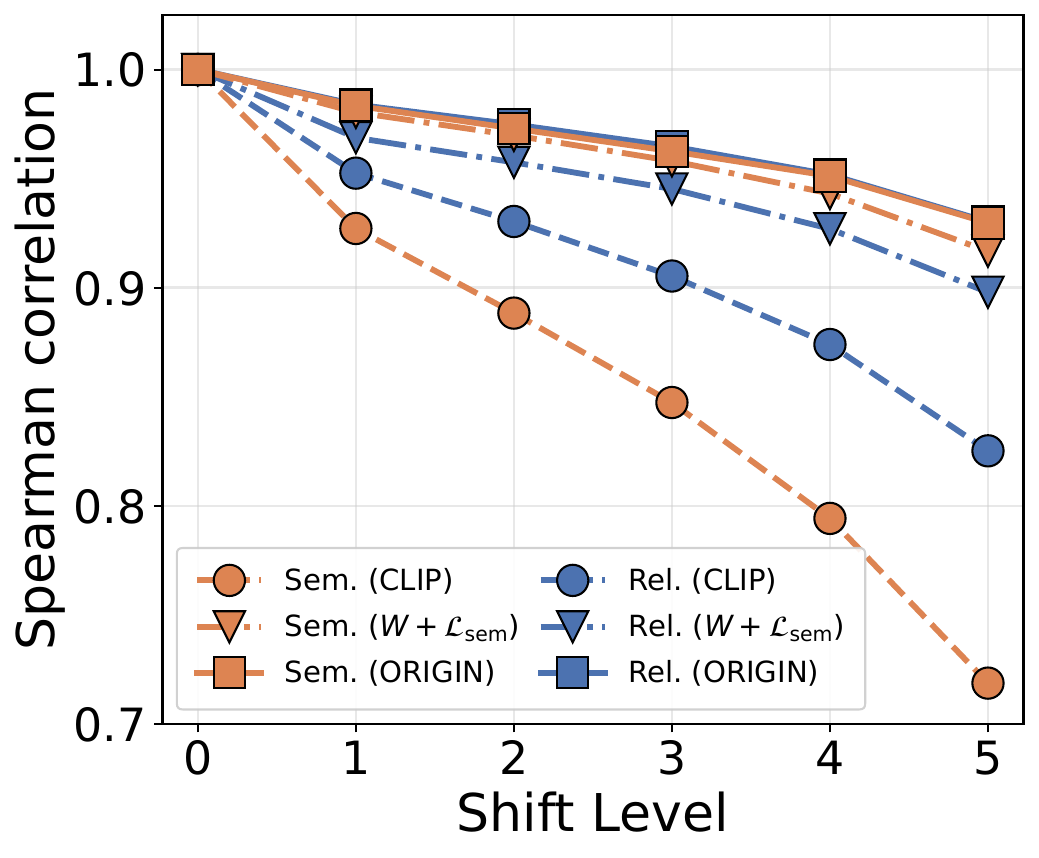}
    \vspace{-10pt}
    \captionof{figure}{Recovery of cross-modal semantics and intra-modal relations under increasing shift levels.}
    \label{fig:exp_recover}
\end{minipage}
\par\medskip

We further investigate the effectiveness of \algname\ in recovering the cross-modal semantics and intra-modal relational structures under distribution shifts
using the Spearman's rank correlation coefficient defined in Eq~\eqref{eq:spearman}, similar to the settings of Figure~\ref{fig:spearman-panel}.
Specifically, we compare the evolving trends of Spearman correlation of both cross-modal semantics (\texttt{sem.}) and intra-modal relations (\texttt{rel.}) along increasing levels of distribution shifts,
using the unadapted raw CLIP model (\texttt{CLIP}),
an adapted CLIP model by a variant of \algname\ that uses the Wasserstein formulation and the semantic InfoNCE loss only (\texttt{W+$\mathcal{L}_{\text{sem}}$}),
and an adapted CLIP model by \algname (\texttt{\algname}).
Formal definition of \texttt{sem.} and \texttt{rel.} are as follows.
\begin{equation}
    \begin{aligned}
        \texttt{Sem.} & =\varrho(\mathbf{V}^{(0)\top}\mathbf{T}^{(0)}, \mathbf{V}^{(s)\top}\mathbf{T}^{(0)}) \\
        \texttt{Rel.} & =\varrho(\mathbf{V}^{(0)\top}\mathbf{V}^{(0)}, \mathbf{V}^{(s)\top}\mathbf{V}^{(s)})
    \end{aligned}
\end{equation}
where $\mathbf{V}^{(0)}$ and $\mathbf{V}^{(s)}$ are the image embeddings of the original and corrupted images outputted by a CLIP model at shift level $s$,
and $\mathbf{T}^{(0)}$ is the text embeddings outputted by the same CLIP model.
The results are shown in Figure~\ref{fig:exp_recover},
which shows that \textbf{(1) \algname\ effectively recovers both the cross-modal semantics and intra-modal relations under distribution shifts},
with the mildest drop in the spearman correlation under increasing levels of distribution shifts.
\textbf{(2) The Wasserstein variant of \algname\ is less effective in recovering the intra-modal relations.}
While $\texttt{W+$\mathcal{L}_{\text{sem}}$}$ recovers cross-modal semantics effectively,
it becomes less effective in recovering the intra-modal relations compared to \algname\ without explicitly modeling the relational structures via the Gromov-Wasserstein formulation.
\textbf{(3) \algname\ even recovers the inter-modal semantics better than its Wasserstein variant},
demonstrating that explicit modeling of the relational structures across different modalities can also improve the recovery of cross-modal semantics under distribution shifts,
which directly determines the performance of VLM TTA.

\subsection{Sensitivity Study of Shift Levels}

\begin{table*}[t]
    \centering
    \begingroup
    \definecolor{benchfirst}{HTML}{F5C6C6}
    \definecolor{benchsecond}{HTML}{FAD6AC}
    \newcommand{\shiftFirst}[1]{\cellcolor{benchfirst}#1}
    \newcommand{\shiftSecond}[1]{\cellcolor{benchsecond}#1}
    \setlength{\tabcolsep}{2.5pt}
    \renewcommand{\arraystretch}{1.05}
    \caption{Sensitivity to shift levels with ViT-B-16. The 1st/2nd best results are shaded \colorbox{benchfirst}{\strut red}/\colorbox{benchsecond}{\strut orange}, respectively.}
    \vspace{-10pt}
    \label{tab:exp_shift_levels}
    \small
    \begin{tabular*}{\textwidth}{@{\extracolsep{\fill}}l*{15}{c}@{}}
        \toprule
        \textbf{Method} & \multicolumn{5}{c}{\textbf{CIFAR10-C}} & \multicolumn{5}{c}{\textbf{CIFAR100-C}} & \multicolumn{5}{c}{\textbf{ImageNet-C}} \\
        \cmidrule(lr){2-6}\cmidrule(lr){7-11}\cmidrule(lr){12-16}
        \textbf{Shift Levels} & \textbf{1} & \textbf{2} & \textbf{3} & \textbf{4} & \textbf{5} & \textbf{1} & \textbf{2} & \textbf{3} & \textbf{4} & \textbf{5} & \textbf{1} & \textbf{2} & \textbf{3} & \textbf{4} & \textbf{5} \\
        \midrule
        BATCLIP & 88.24 & 85.90 & 83.40 & 79.15 & 73.87 & \shiftSecond{62.72} & 58.65 & 54.27 & 49.39 & 42.18 & \shiftSecond{54.12} & \shiftSecond{48.86} & \shiftSecond{45.18} & 38.74 & 30.24 \\
        SubTTA & \shiftSecond{89.56} & \shiftSecond{87.21} & \shiftSecond{84.84} & \shiftSecond{80.88} & 75.95 & 62.65 & 58.93 & 55.38 & 50.95 & 44.28 & 53.34 & 47.76 & 44.32 & \shiftSecond{38.91} & \shiftSecond{31.86} \\
        SAT & 88.37 & 86.21 & 83.96 & 80.61 & \shiftSecond{77.97} & \shiftFirst{63.51} & \shiftSecond{60.08} & \shiftSecond{56.71} & \shiftSecond{53.59} & \shiftSecond{48.43} & 49.21 & 42.83 & 38.03 & 32.38 & 26.15 \\
        \algname & \shiftFirst{89.78} & \shiftFirst{88.24} & \shiftFirst{86.54} & \shiftFirst{83.98} & \shiftFirst{81.02} & 62.68 & \shiftFirst{60.24} & \shiftFirst{57.60} & \shiftFirst{54.41} & \shiftFirst{50.58} & \shiftFirst{54.29} & \shiftFirst{49.79} & \shiftFirst{46.26} & \shiftFirst{40.45} & \shiftFirst{33.36} \\
        \bottomrule
    \end{tabular*}
    \endgroup
    \vspace{-10pt}
\end{table*}

We evaluate how \algname\ performs under different levels of distribution shifts on the three benchmarks,
which all contains 5 levels of corruption severity from Level 1 to Level 5.
The results are included in Table~\ref{tab:exp_shift_levels},
which shows that \textbf{(1) OT improves the robustness of VLM TTA under increasing levels of distribution shifts.}
While all four VLM TTA methods show relatively similar performance under the slightest distribution shift (Level 1),
the performance gap between OT-based methods (\algname\ and SAT) and the other two baselines (BATCLIP and SubTTA) gradually becomes larger as the severity of distribution shifts increases,
highlighting the effectiveness of OT in improving the robustness of VLM TTA under severe distribution shifts.
\textbf{(2) \algname\ remain the most robust methods across different shift levels.}
We can see that \algname\ consistently outperforms all baselines across all shift levels on all three benchmarks,
while achieveing the slightest performance drop under increasing distribution shifts,
demonstrating that preserving the relational structure across modalities via OT and relational constrastive learning improves
the robustness of VLM TTA under severe distribution shifts.

\subsection{Sensitivity Study of Class Templates}
We compare the sensitivity of \algname\ to the different class templates used for generating text embeddings to understand the robustness of \algname\ distribution shifts
happened on the text modality. Specifically, we test the performance of \algname\ and three best-performing baselines, BATCLIP~\citep{maharana2025batclip}, SubTTA~\citep{zeng2026subspace}, and SAT~\citep{mishra2024semantic}
on the 8 text templates listed in Figure~\ref{fig:exp_templates}.
The results are reported in Table~\ref{tab:exp_res_templates}, which shows that
\textbf{(1) \algname\ is consistently robust to different choices of class templates with the least performance variation.}
On CIFAR10-C, the performance of \algname\ varies only 0.37\% across all templates, while the second most robust baseline, BATCLIP, varies 1.84\%.
\algname\ shows similar robustness on CIFAR100-C and ImageNet-C, with performance variations of 2.35\% and 3.85\%, respectively,
consistently outperforming the second most robust baselines.
This demonstrates that \algname\ is also robust to distribution shifts on the text modality,
thanks to the relational-aware inference and adaptation framework that preserves the relational structures across two modalities.
\textbf{(2) \algname\ consistently achieve the best performance across different class templates.}
Under all 8 templates, \algname\ consistently outperforms all baselines on all three benchmarks,
with an significant average improvement of 5.7\% on CIFAR10-C, 7.2\% on CIFAR100-C, and 2.0\% on ImageNet-C over the best-performing baseline.

\begin{table}[t]
    \centering
    \begingroup
    \definecolor{benchfirst}{HTML}{F5C6C6}
    \definecolor{benchsecond}{HTML}{FAD6AC}
    \newcommand{\templateFirst}[1]{\cellcolor{benchfirst}#1}
    \newcommand{\templateSecond}[1]{\cellcolor{benchsecond}#1}
    \setlength{\tabcolsep}{2.5pt}
    \renewcommand{\arraystretch}{1.15}
    \caption{Sensitivity to class templates with ViT-B-16. The 1st/2nd best results are shaded \colorbox{benchfirst}{\strut red}/\colorbox{benchsecond}{\strut orange}, respectively.}
    \vspace{-5pt}
    \label{tab:exp_res_templates}
    \resizebox{\columnwidth}{!}{%
        \begin{tabular}{@{}cl*{8}{c}@{}}
            \toprule
            \multicolumn{2}{l}{\textbf{Method}} & \textbf{T1} & \textbf{T2} & \textbf{T3} & \textbf{T4} & \textbf{T5} & \textbf{T6} & \textbf{T7} & \textbf{T8} \\
            \midrule
            \multirow{4}{*}{\rotatebox{90}{\scriptsize\textbf{CIFAR10-C}}} & BATCLIP & 73.63 & 74.43 & 74.11 & \templateSecond{73.77} & \templateSecond{74.35} & 74.16 & 72.59 & 72.95 \\
             & SubTTA & \templateSecond{75.54} & \templateSecond{76.14} & \templateSecond{75.41} & 73.57 & 73.74 & \templateSecond{75.47} & \templateSecond{75.74} & \templateSecond{75.66} \\
             & SAT & 68.18 & 69.04 & 68.30 & 68.14 & 68.52 & 70.39 & 68.94 & 68.84 \\
             & \algname & \templateFirst{81.06} & \templateFirst{81.10} & \templateFirst{81.05} & \templateFirst{80.80} & \templateFirst{80.73} & \templateFirst{80.95} & \templateFirst{80.82} & \templateFirst{80.88} \\
            \midrule
            \multirow{4}{*}{\rotatebox{90}{\scriptsize\textbf{CIFAR100-C}}} & BATCLIP & 41.71 & 42.03 & 41.26 & 36.61 & 40.66 & 41.60 & 39.04 & 40.48 \\
             & SubTTA & \templateSecond{44.06} & \templateSecond{44.33} & \templateSecond{42.95} & \templateSecond{40.32} & \templateSecond{42.38} & \templateSecond{42.87} & \templateSecond{40.87} & \templateSecond{42.67} \\
             & SAT & 40.70 & 41.06 & 39.50 & 36.42 & 39.23 & 39.73 & 38.01 & 39.22 \\
             & \algname & \templateFirst{50.59} & \templateFirst{49.50} & \templateFirst{50.23} & \templateFirst{48.24} & \templateFirst{49.72} & \templateFirst{50.12} & \templateFirst{49.54} & \templateFirst{50.27} \\
            \midrule
            \multirow{4}{*}{\rotatebox{90}{\scriptsize\textbf{ImageNet-C}}} & BATCLIP & 30.66 & 29.45 & 30.27 & 26.77 & 30.34 & \templateSecond{30.37} & 29.15 & 30.36 \\
             & SubTTA & \templateSecond{31.79} & \templateSecond{32.27} & \templateSecond{30.87} & \templateSecond{28.09} & \templateSecond{31.68} & 29.47 & \templateSecond{29.22} & \templateSecond{31.55} \\
             & SAT & 27.44 & 26.49 & 27.87 & 23.58 & 27.73 & 26.61 & 26.25 & 27.29 \\
             & \algname & \templateFirst{33.12} & \templateFirst{35.48} & \templateFirst{32.74} & \templateFirst{31.63} & \templateFirst{32.39} & \templateFirst{32.94} & \templateFirst{31.75} & \templateFirst{31.88} \\
            \bottomrule
        \end{tabular}%
    }
    \endgroup
    \vspace{-10pt}
\end{table}

\section{Limitations \& Future Works}\label{sec:limit}
While \algname\ demonstrates strong performance, we acknowledge a few limitations. Firstly, \algname\ is an online VLM TTA method which is not directly applicable to strictly episodic TTA settings without maintaining an explicit memory for historical samples. Secondly, like most training-based VLM TTA methods, \algname\ requires access to model gradients for adapting VLMs at test time, which may limits its applicability to closed-source proprietary models. For future works, it would be interesting to discover episodic extension of \algname\ by leveraging efficient OT variants, e.g., sliced OT~\citep{tanguy2025sliced}, for incremental updates of the transport plan, or extending the OT framework to black-box adaption settings.

\end{document}